\documentclass{article}

\PassOptionsToPackage{sort,square,numbers}{natbib}

\usepackage{amsmath}
\usepackage{xspace}
\usepackage{amssymb}
\usepackage{bm}
\usepackage{acronym}
\usepackage{amsthm}
\usepackage{tikz}
\usepackage{svg}
\usepackage{algorithm}%
\usepackage[noend]{algpseudocode}%
\usepackage{dsfont}
\usepackage{enumitem}
\usepackage{comment} %
\usepackage{etoolbox}
\usepackage{subcaption}

\newtheorem{theorem}{Theorem}
\newtheorem{corollary}{Corollary}

\theoremstyle{definition}
\newtheorem{definition}{Definition}

\DeclareMathOperator*{\argmin}{arg\,min}

\acrodef{BO}[BO]{Bayesian optimization}
\acrodef{HPO}[HPO]{hyperparameter optimization}
\acrodef{TR}[TR]{trust region}
\acrodef{EI}[EI]{expected improvement}
\acrodef{TS}[TS]{Thompson sampling}
\acrodef{GP}[GP]{Gaussian process}
\acrodefplural{GP}[GPs]{Gaussian processes}
\acrodef{HDBO}[HDBO]{high-dimensional Bayesian optimization}
\acrodef{VAE}[VAE]{variational autoencoder}

\algrenewcommand\algorithmicrequire{\textbf{Input:}}
\algrenewcommand\algorithmicensure{\textbf{Output:}}

\newcommand{\thesbo}{\textsc{BAxUS}\xspace}
\newcommand{\hesbo}{\textsc{HeSBO}\xspace}
\newcommand{\turbo}{\textsc{TuRBO}\xspace}
\newcommand{\alebo}{\textsc{Alebo}\xspace}
\newcommand{\rembo}{\textsc{REMBO}\xspace}
\newcommand{\saasbo}{\textsc{SAASBO}\xspace}
\newcommand{\cmaes}{\textsc{CMA-ES}\xspace}
\newcommand{\lassobench}{\textsc{LassoBench}\xspace}
\newcommand{\T}{{\intercal}}
\newcommand{\Sparseembedding}{Sparse embedding\xspace}
\newcommand{\sparseembedding}{sparse embedding\xspace}

\newcommand{\sparseembeddings}{sparse embeddings\xspace}

\newcommand{\mlabel}[1]{\label{\thechapter@@#1}}
\DeclareRobustCommand{\mref}[1]{\ref{\thechapter@@#1}}

\usetikzlibrary{matrix,decorations.pathreplacing,calligraphy,arrows,positioning,shapes,backgrounds,fit,tikzmark}

\definecolor{arrowred}{HTML}{E66100}
\definecolor{arrowblue}{HTML}{5D3A9B}

\newbool{eddyproof}
\booltrue{eddyproof}

\usepackage[final]{neurips_2022}

\usepackage[utf8]{inputenc} %
\usepackage[T1]{fontenc}    %
\usepackage[hidelinks]{hyperref}       %
\usepackage{url}            %
\usepackage{booktabs}       %
\usepackage{amsfonts}       %
\usepackage{nicefrac}       %
\usepackage{microtype}      %
\usepackage{xcolor}
\usepackage{mathtools}         %

\title{Increasing the Scope as You Learn: \\Adaptive Bayesian Optimization in Nested Subspaces}

\author{%
  Leonard Papenmeier\\
  Lund University\\
  \texttt{leonard.papenmeier@cs.lth.se} \\
  \And
  Luigi Nardi \\
  Lund University, Stanford University, DBtune \\
  \texttt{luigi.nardi@cs.lth.se} \\
  \And
  Matthias Poloczek\thanks{The work was done before Matthias joined Amazon.} \\
  Amazon \\
  San Francisco, CA 94105, USA \\
  \texttt{matpol@amazon.com} \\
}

\begin{document}

\mathtoolsset{showonlyrefs}

\maketitle

\begin{abstract}
Recent advances have extended the scope of \ac{BO} to expensive-to-evaluate black-box functions with dozens of dimensions, aspiring to unlock impactful applications, for example, in the life sciences, neural architecture search, and robotics. 
However, a closer examination reveals that the state-of-the-art methods for \ac{HDBO} suffer from degrading performance as the number of dimensions increases or even risk failure if certain unverifiable assumptions are not met.
This paper proposes \thesbo that leverages a novel family of nested random subspaces to adapt the space it optimizes over to the problem. 
This ensures high performance while removing the risk of failure, which we assert via theoretical guarantees.
A comprehensive evaluation demonstrates that \thesbo achieves better results than the state-of-the-art methods for a broad set of applications.
\end{abstract}

\section{Introduction}\label{sec:introduction}
The optimization of expensive-to-evaluate black-box functions where no derivative information is available has found many applications, for example, in chemical engineering~\cite{lobato2017parallel,shields2021bayesian,hase2018phoenics,schweidtmann2018machine,burger2020mobile}, materials science~\cite{UENO201618,Frazier2016,packwood2017bayesian,souza2019deepfreak,herbol2018efficient,hase2021gryffin,hughes2021tuning}, aerospace engineering~\cite{lukaczyk2014active,lam2018advances}, hyperparameter optimization~\cite{snoek2012practical,pmlr-v54-klein17a,NIPS2011_86e8f7ab,hvarfner2022pibo}, neural architecture search~\cite{NEURIPS2018_f33ba15e,ru2021interpretable}, vehicle design~\cite{jones2008large,10.1007/s10898-018-0641-2}, hardware design~\cite{nardi2019practical,ejjeh2022hpvm2fpga}, drug discovery~\cite{negoescu2011the},
robotics~\cite{lizotte2007automatic,calandra2016bayesian,rai2018bayesian,tuprints5878,mayr2022skill},
and the life sciences~\cite{tallorin2018discovering,vsehic2021lassobench,cosenza2022multi}.
Here increasing the number of dimensions (or parameters) of the optimization problem usually allows for better solutions. 
For example, by exposing more process parameters of a chemical reaction, we obtain a more granular control of the process; for a design task, we may optimize a larger number of design decisions jointly; in robotics, we gain access to more sophisticated control policies.

A series of breakthroughs have recently pushed the envelope of \acl{HDBO} and facilitated a wider adoption in science and engineering. 
The key challenge for further scaling is the so-called curse of dimensionality.
The complexity of the task of finding an optimum grows exponentially with the number of dimensions~\citep{binois2022survey,eriksson2019scalable}.
Recently, methods that rely on local 'trust regions' have gained popularity. 
They usually achieve good performance for problems with up to a couple of dozen input parameters. 
However, we observe that their performance degrades for higher-dimensional problems. 
This is not surprising, given that trust regions have a smaller volume but still the full dimensionality of the problem.
Other state-of-the-art methods suppose the existence of a low-dimensional active subspace and enjoy great scalability if they find such a space. 
The caveat is that its existence is usually not known for practical applications. 
Moreover, the user needs to `guess' a good upper bound on its dimensionality to enjoy a good sample efficiency.

In this work, we propose a theoretically-founded approach for high-dimensional Bayesian optimization, \thesbo (\textbf{B}ayesian optimization with \textbf{a}daptively e\textbf{x}panding s\textbf{u}bspace\textbf{s}), that reliably achieves a high performance on a comprehensive set of applications.
\thesbo utilizes a family of nested embedded spaces to increase the dimensionality of the domain that it optimizes over as it collects more data.
As a byproduct, \thesbo can leverage an active subspace, if it exists, without requiring the user to 'guess' its dimensionality.
\thesbo is based on a novel random linear subspace embedding that enables a more efficient optimization and has strong theoretical guarantees.
We make the following contributions: 
\begin{enumerate}[leftmargin=*]
    \item We develop \thesbo that reliably achieves excellent solutions on a broad set of high-dimensional tasks, outperforming the state-of-the-art.
    \item We present a novel family of nested random embeddings that has the following properties:
    a) The \thesbo embedding provides a larger worst-case guarantee for containing a global optimum than the \hesbo embedding proposed by \cite{nayebi2019a}.
    b) The \thesbo embedding is an optimal \emph{\sparseembedding}, as defined in Def.~\ref{def:count-sketch-type}.
    c) Its probability of containing an optimum converges to the one of the \hesbo embedding as the input dimensionality $D\rightarrow \infty$.
    \item We conduct a comprehensive evaluation on a representative collection of benchmarks that demonstrates that \thesbo outperforms the state-of-the-art methods.
\end{enumerate}

The remainder of this paper is structured as follows.
Section \ref{sec:background_and_related_work} states the problem and discusses related work.
Section \ref{sec:algorithm} presents the \thesbo algorithm and the corresponding embedding.
Section \ref{sec:results} evaluates \thesbo on a variety of benchmarks.
We give concluding remarks in Section~\ref{sec:discussion}.

\section{Background}\label{sec:background_and_related_work}
The task is to find a minimizer
\[
\bm{x}^*\in \argmin_{\bm{x}\in \mathcal{X}} f(\bm{x}),
\] 
where $\mathcal{X} = [-1,+1]^D$.
The objective function $f: \mathcal{X}\rightarrow \mathbb{R}$ is an expensive-to-evaluate black-box function.
Hence the number of function evaluations needed to find an optimizer is crucial.
Evaluations may be subject to observational noise, i.e.,  $f(\bm{x}_i)+\varepsilon_i$, with $\varepsilon_i\sim \mathcal{N}(0,\sigma^2)$. 
This work focuses on \emph{scalable high-dimensional Bayesian optimization}, where the \emph{input dimensionality}~$D$ is in the hundreds, and the sampling budget may comprise a thousand or more function evaluations. 

\paragraph{Linear embeddings.}
A successful approach for \ac{HDBO} is to assume the existence of an \emph{active subspace}~\cite{constantine2015active}, i.e., there exist a space $\mathcal{Z}\subseteq \mathbb{R}^{d_e}$, with $d_e\leq D$ and a function $g: \mathcal{Z}\rightarrow \mathbb{R}$, such that for all $\bm{x}$: $g(T\bm{x}) = f(\bm{x})$ where $T\in \mathbb{R}^{d_e\times D}$ is a projection matrix projecting $\bm{x}$ onto $\mathcal{Z}$ and $d_e$ is the \emph{effective dimensionality} of the problem.
In practice, both $d_e$ and $\mathcal{Z}$ are unknown.

\rembo (Random embedding \ac{BO})~\cite{wang2016bayesian} and \hesbo (Hashing-enhanced subspace \ac{BO})~\cite{nayebi2019a} try to capture this active subspace by a randomly chosen linear subspace.
Therefore, they generate a random projection matrix~$S^\T$ that maps from a $d$-dimensional subspace~$\mathcal{Y}\in \mathbb{R}^d$ with $d\ll D$ (the \emph{target space})  to $\mathcal{X}$.
We call~$d$ the \emph{target dimensionality}.
For \rembo, each entry in~$S^\T$ is normally distributed. 
\rembo uses a heuristic to determine a hyperrectangle in~$\mathcal{Y}$ that it optimizes over. 
Note that the bounded domain may not contain a point that maps to an optimizer of~$f$, a risk aggravated by distortions introduced by the projection.
\citep{binois2015uncertainty, binois2015warped, binois2020choice} proposed ideas to mitigate the issue.
\hesbo's random projection assigns one target dimension and sign ($\pm 1$) to each input dimension.
This embedding is inspired by the count-sketch algorithm~\cite{charikar2002finding} for estimating frequent items in data streams.
The sparse projection matrix $S^\T$ is binary except for the signs, and each row has exactly one non-zero entry.
Even though this embedding avoids \rembo's distortions, as the authors proved, it has a lower probability of containing the optimum~\cite{alebo}.
\alebo~\citep{alebo} uses a Mahalanobis kernel and imposes linear constraints on the acquisition function to avoid projecting outside of $\mathcal{X}$. 

\paragraph{Non-linear embeddings.}
Several works use autoencoders to learn non-linear spaces for optimization, trading in sample efficiency.
\citet{tripp2020sample} change the training objective of a variational autoencoder (VAE)~\citep{Kingma2014} to make the target space more suitable for optimization.
They give higher weight to better-performing points when training the \ac{VAE} and show that this improves optimization.
\citet{moriconi2019high} incorporate the training of an autoencoder directly into the likelihood maximization of a \ac{GP} surrogate. The computational cost is cubic in the number of samples and the input dimension.
\citet{lu2018structured} and \citet{maus2022local} use autoencoders to learn embeddings of highly structured input spaces such as kernels or molecules.
Other approaches include partial least squares~\cite{bouhlel2018efficient} or sliced inverse regression~\cite{chen2020semi}. 

\paragraph{High-dimensional BO in the input space.} 
A popular approach to make \ac{HDBO} in the input space feasible is \acp{TR}~\cite{regis2016trust,pedrielli2016g,eriksson2019scalable,zhou2020trust}.
The \turbo algorithm~\citep{eriksson2019scalable} optimizes over bounded \acp{TR} instead of the global space, adapting their side lengths and the center points during the optimization process.
By restricting function evaluations to trust regions, \turbo addressed the problem of over-exploration;
see~\citep{eriksson2019scalable} for details.
Note that the \acp{TR} have full input dimensionality, which may impact \turbo's ability to scale to very large dimensions.
Nonetheless, \turbo set a new state-of-the-art by scaling to dozens on input dimensions and thousands of function evaluations.  
\citet{wan2021think} extended the idea of \acp{TR} to categorical and mixed spaces by using the Hamming distance to define the \ac{TR} boundaries.
\saasbo~\cite{saasbo} uses sparse priors on the \ac{GP} length scales which seems particularly valuable if the active subspace is axis-aligned. 
Indeed, \saasbo can outperform \turbo on certain benchmarks~\citep{saasbo}.
The cost of inference scales cubically with the number of function evaluations;
thus, \saasbo is not expected to scale beyond small sampling budgets, which is confirmed by our experiments.
Another line of research relies on the assumption that the input space has an additive structure ~\cite{kandasamy2015high,gardner2017discovering,wang2018batched,mutny2018efficient}.
Additive \acp{GP} rely on computationally expensive sampling methods to learn a decomposition of the input variables, which limits the scalability of such methods to problems of moderate dimensionalities and sampling budgets~\cite{nayebi2019a,eriksson2019scalable}.
\citet{wang2020learning} combined the meta-level algorithm \textsc{LA-MCTS} with \turbo to improve optimization performance by learning a hierarchical space partition.

In Sect.~\ref{sec:results} we evaluate the performances of \turbo, \saasbo, \alebo, and \hesbo.
Moreover, we study the popular \cmaes~\citep{cma-es} and random search~\citep{bergstra2012random}.

\section{The \thesbo algorithm}\label{sec:algorithm}
\citet{wang2016bayesian} showed that the \rembo embedding contains an optimum in the target space with probability one if $d\geq d_e$ and if there are no bounds on the target and input spaces, i.e.,  $\mathcal{Y}=\mathbb{R}^d$ and $\mathcal{X}=\mathbb{R}^D$.
For $d<d_e$, it is in general impossible to represent an optimum in $\mathcal{X}$ for arbitrary~$f$ because $S$ projects to a $d$-dimensional subspace in $\mathcal{X}$.
We call the probability of a target space to contain the optimum the \emph{success probability}.
For $d \geq d_e$, there is a positive success probability that increases with $d$~\cite{alebo,wang2016bayesian}.
The main problem is to set $d$ sufficiently small to avoid the detrimental effects of the curse of dimensionality, while keeping it as large as necessary to achieve a high probability that $\mathcal{Y}$ contains an optimum.

In practice, the active subspace and its dimensionality are usually unknown.
The performance of methods such as \rembo~\cite{wang2016bayesian}, \hesbo~\cite{nayebi2019a}, and \alebo~\cite{alebo} depends on choosing $d$ such that the success probability is high.
Therefore, they implicitly rely on guessing the effective dimensionality $d_e$ appropriately.
We argue that choosing the target dimensionality is problematic in many practical applications.
If chosen too small, the subspace cannot represent $f$ sufficiently well.
If it is chosen too large, the curse of dimensionality slows down the optimization.

\begin{algorithm}[t]
\caption{\thesbo}
\label{alg:thesbo2}
\begin{algorithmic}[1]
    \Require $b$: new bins per dimension split, $D$: input dimension, $m_D$: \# evaluations by which the input dimension is reached.
    \Ensure minimizer $\bm{x}^*\in \argmin_{\bm{x}\in\mathcal{X}}f(\bm{x})$.
    \State $d_0$ $\gets$ initial target dimensionality of the subspace given by Eq.~\eqref{eq:dinit}.
    \State Compute initial projection matrix $S^\T: \mathcal{Y}\rightarrow \mathcal{X}$ by \thesbo embedding for $D$ and $d_0$.
    \State Sample initial data $\mathcal{D} = \{(\bm{y}_1,f(S^\T\bm{y}_1)), \ldots\}$ and fit the \ac{GP} surrogate.
    \State $n \gets 0$
    \While{evaluation budget not exhausted}
        \State $L \gets L_{\textrm{init} }$ \Comment{Initialize the trust region}
        \State Calculate number of accepted ``failures'' as described in Sec.~\ref{subsec:controlling_fail_tolerance}: $\tau_{\textrm{fail} }^s \gets \max \left (1,\min \left ( \left \lfloor \frac{m_i^s}{k} \right \rfloor, d_n \right )\right )$
        \While{$L>L_{\min}$ \textbf{and} evaluation budget not exhausted}
            \State Find $\bm{y}$ by Thompson sampling in \acs{TR}, evaluate $f(S^\T\bm{y})$, and add to $\mathcal{D}$.
            \State Re-fit the GP hyperparameters.
            \State Adjust trust region, see Section~\ref{ssec:tr_approach} for details.
        \EndWhile
        \If{$d_n<D$}
            \State $d_{n+1} \gets \min(d_n\cdot{}(b+1), D)$.
             \State Increase $S^\T$ by Algorithm~\ref{alg:opei}.\Comment{See Appendix~\ref{app:opei}}
        \Else
            \State Re-sample initial data and discard previous observations, $d_{n+1} \gets d_n$.
        \EndIf
        \State $n \gets n+1$
    \EndWhile
    \State \textbf{Return} $S^\T(\argmin_{\bm{y}\in \mathcal{D}} f(S^\T\bm{y}))$ or $S^\T(\argmin_{\bm{y}\in \mathcal{D}} \mathbb{E}_n[f(S^\T \bm{y})])$. \Comment{Return the best observation in case of observations without noise, or the best point according to posterior mean in case of noisy observations.}
\end{algorithmic}
\end{algorithm}

The proposed algorithm, \thesbo, operates on target spaces of increasing dimensionality while preserving previous observations.
Let $d_{\textrm{init} }$ be the initial target dimensionality and $m$ the total evaluation budget.
\thesbo starts with a $d_{\textrm{init} }$-dimensional embedding that is increased over time until, after $m_D\leq m$ evaluations, it roughly reaches the input dimensionality $D$.
With this strategy, we can leverage the efficiency of \ac{BO} in low-dimensional spaces while guaranteeing to find an optimum in the limit.
Increasing the target dimensionality is enabled by a novel embedding, which lets us carry over observations from previous, lower-dimensional target spaces into more high-dimensional target spaces.
We further use a \ac{TR}-based approach based on \citet{eriksson2019scalable} to carry out optimization for high target dimensions effectively.
\thesbo uses a GP surrogate~\cite{williams2006gaussian} to model the function in the target space.
Algorithm~\ref{alg:thesbo2} gives the pseudocode for \thesbo.
In Appendix~\ref{app:global_convergence}, we prove global convergence for \thesbo.
We will now present the different components in detail.

\subsection{The sparse \thesbo subspace embedding}\label{ssec:embedding}
The \thesbo embedding uses a sparse projection matrix to map from $\mathcal{Y}$ to $\mathcal{X}$.
The number of non-zero entries in this matrix is equal to the input dimensionality $D$.
Another embedding with this property is the \hesbo embedding~\cite{nayebi2019a}.
Given the $D$ and a target dimensionality $d$, \hesbo samples a target dimension in $\{1,\ldots,d\}$ and a sign $\{\pm 1\}$ for each input dimension uniformly at random.
Conversely, each target dimension has a set of signed contributing input dimensions.
We call the set of contributing input dimensions to a target dimension a \emph{bin}.
These relations implicitly define the embedding matrix $S\in \{0,\pm1\}^{d\times D}$, where each column has exactly one non-zero entry~\cite{charikar2002finding}.
In the \hesbo embedding, the number of contributing input dimensions varies between 0 and $D$.

The interpretation of contributing input dimensions allows for an intuitive way to refine the embedding, which is shown in Figure~\ref{fig:splitting}.
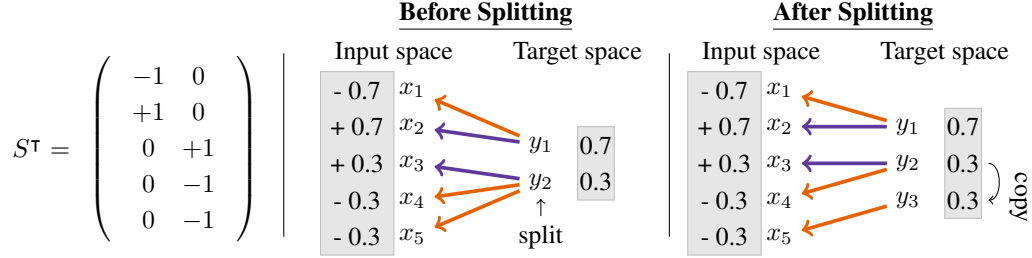
\begin{figure}
    \centering
    \resizebox{\textwidth}{!}{
        \begin{tikzpicture}

\node (ST) at (-4.8,-1.3) {$S^\T=$};
\matrix at (-3, -1.3) [matrix of math nodes,left delimiter=(,right delimiter=)](M){ 
    -1 & 0\\
    +1 & 0\\
    0 & +1\\
    0 & -1\\
    0 & -1\\
   };
\draw[-] (-1.5,0) -- (-1.5,-2.5);

\node (A) at (0,0) {Input space};
\node (BFS) at (1.25,0.5) {\underline{\textbf{Before Splitting}}};
\node (B) at (2.5,0) {Target space};
\draw[draw=gray!50,fill=gray!20] (-1,-0.25) rectangle ++(1,-2.5);
\node (x1val1) at (-0.5,-0.5) {- 0.7};
\node (x2val1) at (-0.5,-1) {+ 0.7};
\node (x3val1) at (-0.5,-1.5) {+ 0.3};
\node (x4val1) at (-0.5,-2) {- 0.3};
\node (x5val1) at (-0.5,-2.5) {- 0.3};

\node (x1) at (0.25,-0.5) {$x_1$};
\node (x2) at (0.25,-1) {$x_2$};
\node (x3) at (0.25,-1.5) {$x_3$};
\node (x4) at (0.25,-2) {$x_4$};
\node (x5) at (0.25,-2.5) {$x_5$};

\node (y1) at (2,-1.25) {$y_1$};
\node (y2) at (2,-1.75) {$y_2$};
\node (split) at (2, -2.5) {split};
\draw[->] (split) -- (y2);

\draw[draw=gray!50,fill=gray!20] (3,-1) rectangle ++(-0.5,-1);
\node (y1val1) at (2.75, -1.25) {0.7};
\node (y2val1) at (2.75, -1.75) {0.3};

\draw[->,arrowred,line width=0.5mm] (y1) -- (x1);
\draw[->,arrowblue,line width=0.5mm] (y1) -- (x2);
\draw[->,arrowblue,line width=0.5mm] (y2) -- (x3);
\draw[->,arrowred,line width=0.5mm] (y2) -- (x4);
\draw[->,arrowred,line width=0.5mm] (y2) -- (x5);

\draw[-] (3.75,0) -- (3.75,-2.5);

\node (C) at (5,0) {Input space};
\node (AFS) at (6.25,0.5) {\underline{\textbf{After Splitting}}};
\node (D) at (7.5,0) {Target space};
\draw[draw=gray!50,fill=gray!20] (4,-0.25) rectangle ++(1,-2.5);
\node (x1val2) at (4.5,-0.5) {- 0.7};
\node (x2val2) at (4.5,-1) {+ 0.7};
\node (x3val2) at (4.5,-1.5) {+ 0.3};
\node (x4val2) at (4.5,-2) {- 0.3};
\node (x5val2) at (4.5,-2.5) {- 0.3};

\node (x12) at (5.25,-0.5) {$x_1$};
\node (x22) at (5.25,-1) {$x_2$};
\node (x32) at (5.25,-1.5) {$x_3$};
\node (x42) at (5.25,-2) {$x_4$};
\node (x52) at (5.25,-2.5) {$x_5$};

\node (y12) at (7,-1) {$y_1$};
\node (y22) at (7,-1.5) {$y_2$};
\node (y32) at (7,-2) {$y_3$};

\draw[draw=gray!50,fill=gray!20] (8,-0.75) rectangle ++(-0.5,-1.5);
\node (y1val2) at (7.75, -1) {0.7};
\node (y2val2) at (7.75, -1.5) {0.3};
\node (y3val2) at (7.75, -2) {0.3};

\draw[->,arrowred,line width=0.5mm] (y12) -- (x12);
\draw[->,arrowblue,line width=0.5mm] (y12) -- (x22);
\draw[->,arrowblue,line width=0.5mm] (y22) -- (x32);
\draw[->,arrowred,line width=0.5mm] (y22) -- (x42);
\draw[->,arrowred,line width=0.5mm] (y32) -- (x52);

\draw[->] (y2val2.east) to [out=0,in=0] (y3val2.east);
\node[right,sloped,rotate=270] at (8.5,-1.5) {copy};

\end{tikzpicture}
    }
    \caption{Observations are kept when increasing the target dimensionality. We give an example of the splitting method for $D=5$ and $d=2$.
    The first target dimension $y_1$ has two contributing input dimensions, $x_1$ and $x_2$. $y_2$ has three contributing input dimensions, $x_3$, $x_4$, and $x_5$.
    By $S^\T$, a point $(0.7, 0.3)^\T$ in the target space is mapped to $(-0.7, +0.7, +0.3, -0.3, -0.3)^\T$ in the input space.
    Assigning the fifth input dimension to a new target dimension and copying the function values from the second target dimension does not change the observation in the input space.
    The new $S^\T$ is not shown but has one additional column with $-1$ in the last row, and the last row of the second column is set to 0.}
    \label{fig:splitting}
\end{figure}
We update the embedding matrix such that contributing input dimensions of the target dimension are re-assigned to the current bin and $b$ new bins.
We then say that we \emph{split} the corresponding target dimension.
Importantly, this type of embedding allows for retaining observations (see Figure~\ref{fig:splitting}).
Assume for example, that $y_i$ is the dimension to be split.
The contributing input dimensions are re-assigned to $y_i$ and three new target dimensions $y_j$, $y_k$, and $y_l$ (here, $b=3$);
the observations can be retained by copying the value of the coordinate $y_i$ to the coordinates $y_j$, $y_k$, and $y_l$.
Thus, the observations are contained in the old and in the new target space.
Algorithm~\ref{alg:opei} describes the procedure in detail.

In the \thesbo embedding, we force each bin of a target dimension to have roughly the same number of contributing input dimensions:
the bin sizes differ by at most one.
First, we create a random permutation of the input dimensions $1,\ldots,D$.
The list of input dimensions is split into $\min(d,D)$ individual bins.
If $d$ does not divide $D$, not all bins can have the same size.
We split the permutation of input dimensions such that the $i$-th bin has size $\lceil D/d \rceil$, if $i+d\lfloor D/d \rfloor \leq D$, and $\lfloor D/d \rfloor$ otherwise.
The first bins have one additional element with this construction if $d$ does not divide $D$.
We further randomly assign a sign to each input dimension.
The sign of the input dimensions and their assignment to target dimensions then implicitly define $S^\T$ (see Figure~\ref{fig:splitting}).
We now show that the \thesbo embedding has a strictly larger worst-case success probability than the \hesbo embedding.
We establish the following two definitions.
\begin{definition}[\textbf{\Sparseembedding matrix}]
    \label{def:count-sketch-type}
    A matrix $S\in \{0,\pm 1\}^{d\times D}$ is a \emph{\sparseembedding matrix} if and only if each column in $S$ has exactly one non-zero entry~\cite{woodruff2014sketching}.
\end{definition}
We formalize the event of ``recovering an optimum''~\cite{alebo} as follows.
\begin{definition}[\textbf{Success of a \sparseembedding}]
    \label{def:success}
    A success of a random \sparseembedding is the event $Y^*=$ ``All $d_e$ active input dimensions are mapped to distinct target dimensions.''
\end{definition}
It is important to note that the definition of a success is sufficient but not necessary for the embedding to contain a global optimum.
For example, if the origin is a global optimum, then both embeddings contain it with probability one.
In that sense, the above definition provides a \emph{worst-case guarantee}.
We refer to Definition~\ref{def:function-with-active-subspace} in Appendix~\ref{sec:proofs_embedding} for a formal definition of a sparse function.
In Theorem~\ref{lemma:ft_embedding}, we give the worst-case success probability of the \thesbo embedding.
All proofs have been deferred to Appendix~\ref{sec:proofs_embedding}.
Note that other than in the count-sketch algorithm~\cite{charikar2002finding}, our hashing function is not pairwise independent.
However, this does not affect our theoretical analysis.

\begin{theorem}[\textbf{Worst-case success probability of the \thesbo embedding}]
    \label{lemma:ft_embedding}
    Let $D$ be the input dimensionality and $d\geq d_e$ the dimensionality of the embedding.
Let $\beta_{\textrm{small} } = \left \lfloor \frac{D}{d} \right \rfloor$ and $\beta_{\textrm{large} } = \left \lceil \frac{D}{d} \right \rceil$ be the small and large bin sizes.
Then the probability of $Y^*$ (see Definition~\ref{def:success}) for the \thesbo embedding is
\begin{equation}
    p_B(Y^*; D,d,d_e) = \frac{\sum_{i=0}^{d_e}\binom{d(1+\beta_{\textrm{small} })-D}{i}\binom{D-d\beta_{\textrm{small} }}{d_e-i}\beta_{\textrm{small} }^i\beta_{\textrm{large} }^{d_e-i}}{\binom{D}{d_e}}.\label{eq:binning}
\end{equation}
\end{theorem}
Figure~\ref{fig:success_probabilities} shows the worst-case success probabilities of the \thesbo and \hesbo embeddings for three different settings of $D$.
\begin{figure}[bt]
    \centering
    \resizebox{\textwidth}{!}{\includeinkscape{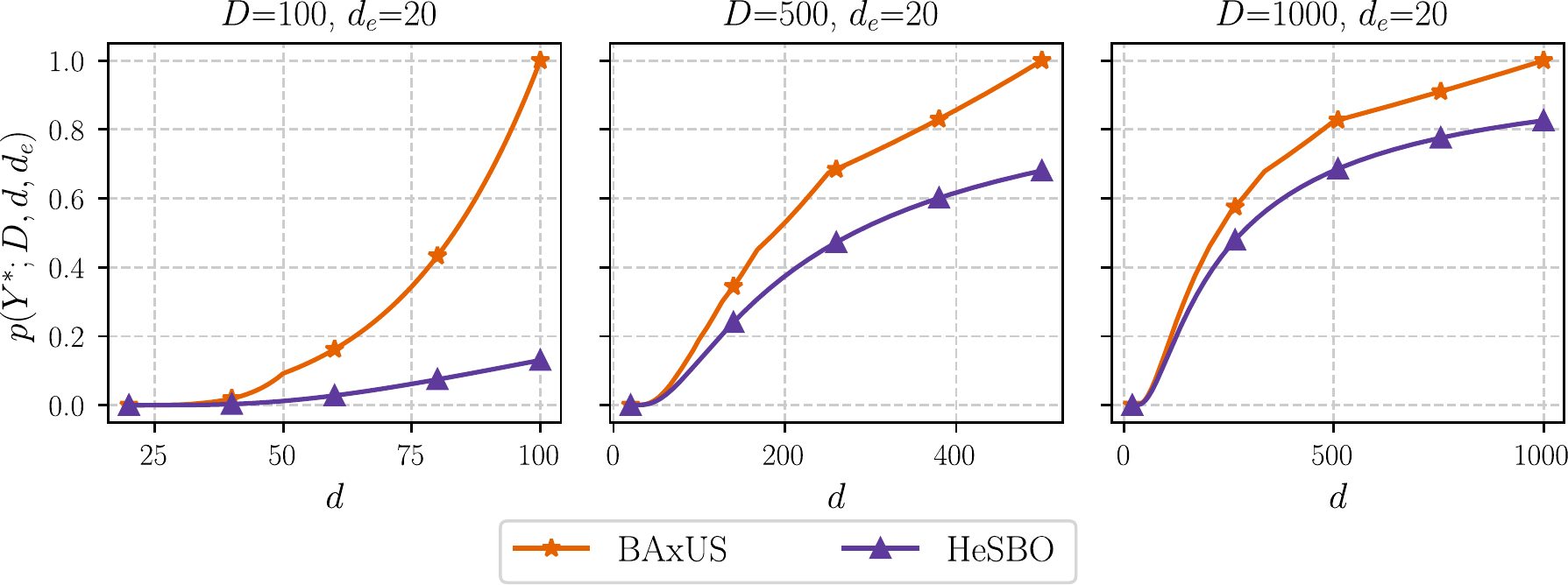}}
    \caption{The worst-case guarantees for the success probabilities $p(Y^*; D,d,d_e)$ of the \thesbo and \hesbo embeddings for different input dimensionalities, $D{=}100$ (left), $D{=}500$ (center), $D{=}1000$ (right), as a function of the target dimensionality~$d$. The effective dimensionality is~$d_e{=}20$.
    The \thesbo embedding has a higher worst-case success probability than the \hesbo embedding.
    The improvement is large for input dimensionalities in the low hundreds and still substantial for 1000D tasks. In accordance with the theoretical analysis, the difference vanishes as the input dimension~$D$ grows.
    }
    \label{fig:success_probabilities}
\end{figure}
The worst-case success probability of the \hesbo embedding is given by $p_H(Y^*;d,d_e)=d!/((d-d_e)!d^{d_e})$~(see~\cite{alebo} and Appendix~\ref{proof:equivalence}).
It is independent of $D$ but is shown for varying $d$-ranges on the $x$-axis, therefore the probabilities seem to change between the different subplots.
The \thesbo embedding ensures that the worst-case success probability is one for~$d = D$.
Discontinuities in the curve of the \thesbo embedding occur due to the unequal bin sizes in the \thesbo embedding's worst-case success probability.
The difference between the two embeddings in Figure~\ref{fig:success_probabilities} is particularly striking when $d_e$ is high:
for example, for~$d_e=20$ \hesbo requires $d=1000$ to reach a worst-case success probability of approximately $0.8$, whereas the \thesbo embedding has a success probability of $1$ as soon as $d=D$.
For finite $D$, \hesbo's worst-case success probability is smaller than \thesbo' success probability.
\begin{corollary}
    \label{corr:equivalence}
    For $D\rightarrow \infty$, the worst-case success probability of the \thesbo embedding is
\begin{equation}
    \lim_{D\rightarrow \infty}p_B(Y^*;D,d,d_e)=\frac{d!}{(d-d_e)!d^{d_e}},
\end{equation}
and hence matches \hesbo's worst-case success probability $p_H(Y^*;d,d_e)$.
\end{corollary}
We show that the \thesbo embedding is optimal among \sparseembeddings.
\begin{corollary}
    \label{corr:optimality}
    With the same input, target, and effective dimensionalities ($D$, $d$, and $d_e$), no \sparseembedding has a higher worst-case success probability than the \thesbo embedding.

\end{corollary}

\subsection{Trust-region approach}\label{ssec:tr_approach}
Similar to \turbo~\cite{eriksson2019scalable}, \thesbo operates in trust regions (\acp{TR}).
\acp{TR} are hyper-rectangles in the input space.
Their shape is determined by their \emph{base side length} $L$ and the \ac{GP} length scales. The side length for each dimension is proportional to the corresponding length scale of the \ac{GP} kernel fitted to the data.
The idea of this construction is that length scales indicate how quickly the function changes along the associated dimension.
Thus, the \ac{TR} is rescaled accordingly.
The volume of a \ac{TR} is shrunk when \turbo fails $\tau_{\textrm{fail} }$ consecutive times to make progress, i.e., to find a better solution.
If the algorithm consecutively makes progress for $\tau_{\textrm{success} }=3$ times, it expands the \ac{TR}.
It restarts when the base side length $L$ of the current \ac{TR} falls below a threshold $L_{\min}\coloneqq2^{-7}$.
In that case, it discards all observations for the \ac{TR} and initializes a new \ac{TR} on new samples.
\Acp{TR} enable \turbo to focus on regions of the space close to the incumbent, i.e., the current best solution found by the algorithm.
To choose the next evaluation point, \turbo uses Thompson sampling~\cite{10.2307/2332286}, i.e., it draws a realization of the \ac{GP} on a set of candidate locations in the \ac{TR} and then selects a point of minimum sampled value.

\Acp{TR} are an essential component of \thesbo because the target dimensionality usually grows exponentially during a run of the algorithm.
We use the same hyperparameter settings as \turbo~\cite{eriksson2019scalable} with the following modifications.
First, we change the criterion for when to restart a \ac{TR}.
Instead of restarting a \ac{TR} when it becomes too small, we increase the target dimensionality by splitting each target dimension into several new bins unless \thesbo has already reached the input dimensionality.
In this case, we reset the \ac{TR} base side length to the initial value and re-initialize the algorithm with a new random set of initial observations.
By resetting the base side length, we also avoid convergence to a particular local minimum as the \ac{TR} covers large regions of the space again.
\turbo solves this problem by allowing for multiple parallel \acp{TR}.
Secondly, we change the number of accepted ``failures'' $\tau_{\textrm{fail} }$, such that \thesbo can roughly reach the input dimensionality in a fixed number of evaluations as described in Section~\ref{par:setting_the_initial_target_dimension}.

\subsection{Splitting strategy}\label{subsec:splitting-strategy}
Starting in a low-dimensional embedded space, \thesbo successively grows the target \mbox{dimensionality} to increase the probability of containing an optimum.
By Corollary~\ref{corr:optimality}, it is optimal to keep the number of contributing input dimensions in the target bins as equal as possible.
At each splitting point, the target dimensionality grows exponentially.
The number of splits required to reach some input dimensionality $D$ is logarithmic in $D$.
\thesbo uses a larger evaluation budget in each split.
Suppose that the algorithm starts in target dimensionality~$d_{\textrm{init} }$.
Then, after $k$ splits, the target dimensionality is $d_k=d_{\textrm{init} }(b+1)^k$.

\subsection{Controlling the number of accepted failures}\label{subsec:controlling_fail_tolerance}
\thesbo needs to be able to reach high target dimensionalities to find a global optimum.
As described in Section~\ref{ssec:tr_approach}, \thesbo increases the target dimensionality when the TR base side length falls below the minimum threshold.
For this to happen, the TR base length needs to be halved at least $k=\left \lfloor \log_{\frac{1}{2}} \frac{L_{\textrm{min} }}{L_{\textrm{init}}}\right \rfloor$ times.
Halving occurs if \thesbo consecutively fails $\tau_{\textrm{fail}}$ times in finding a better function value.
If, similarly to \turbo, we set the number of accepted ``failures'' $\tau_{\textrm{fail}}$ to the current target dimensionality of the TR, we get the lower bound $k\cdot \tau_{\textrm{fail}}$ on the number of function evaluations spent in that target dimensionality.
This bound does not scale with the input dimensionality $D$ of the problem, i.e., the maximum target dimensionality is independent of $D$ for a fixed evaluation budget.

To enable \thesbo to reach any desired target dimensionality for the fixed evaluation budget, we scale down $\tau_{\textrm{fail}}$ dependent on $D$, i.e., we adjust the lower bound.
We choose to make it dependent on $D$ as we are guaranteed that the target space contains all global optima if the final target space corresponds to the input space $\mathcal{X}$ (see Appendix~\ref{app:global_convergence}).
In contrast to imposing a hard limit on the number of function evaluations in a target dimensionality, scaling down the number of accepted ``failures'' has the advantage that we do not restrain \thesbo in cases where it finds better function values.
The idea is to choose $\tau_{\textrm{fail}}$ dependent on the current target dimensionality $d_i$ \emph{and} such that \thesbo can reach any desired target dimensionality.

We calculate the number of splits~$n$ required to reach~$D$ by
\begin{equation}
    D \approx d_{\textrm{init}}\cdot{}(b+1)^n \Rightarrow n=\left \lfloor \log_{b+1}\frac{D}{d_{\textrm{init} }}\right \rceil \label{eq:number_of_splits},
\end{equation}
with $\lfloor \cdot \rceil$ indicating rounding to the nearest integer.
The minimum evaluation budget for a split is then found by multiplying~$m_D$ with the ``weight'' of each target dimensionality.
We assign each split~$i$ a \emph{split budget}~$m_i^s$ that is proportional to~$d_i$, such that $\sum_{i=0}^n m^s_i = m_D$, where~$m_D$ is the budgeted number of function evaluations until~$D$ would be reached under the above assumptions:
\[
    m^s_i= \left \lfloor m_D \frac{d_i}{\sum_{k=0}^n d_k} \right \rceil = \left \lfloor \frac{b\cdot{}m_D\cdot{}d_i}{d_{\textrm{init}}((b+1)^{n+1}-1)}\right \rceil.
\]
Finally, we set the number $\tau_{\textrm{fail}}^i$ of accepted ``failures'' for the~$i$-th target dimensionality~$d_i$ such that (1) it adheres to its \emph{split budget} in the event that it never obtains a better function value, (2) it is not larger than if we would use \turbo's choice $d_i$, and (3) it is at least 1:
\[
    \tau_{\textrm{fail} }^i = \max \left (1,\min \left ( \left \lfloor \frac{m_i^s}{k} \right \rfloor, d_i \right )\right ).
\]

\paragraph{Setting the initial target dimension.}\label{par:setting_the_initial_target_dimension}

Due to the rounding in Eq.~\eqref{eq:number_of_splits} and the exponential growth of the target dimensionality, the final target dimensionality $d_{\textrm{init}}\cdot (b+1)^n$ might differ considerably from the input dimensionality $D$.
This is undesirable as we might not reach $D$ before depleting the evaluation budget, or we might overestimate the evaluation budget for the final target dimensionality~$d_n$.
Therefore, we set the initial target dimensionality such that the final target dimensionality is as close to~$D$ as possible:
\begin{equation}
    d_{\textrm{init} } = \argmin_{i\in \{1,\ldots,b\}} \left |i\cdot{}(b+1)^n - D  \right |, \label{eq:dinit}
\end{equation}
where $n$ is given by Eq.~\eqref{eq:number_of_splits}.
We point out that $1 \leq d_{\textrm{init}} \leq b$.
An alternative to adjusting~$d_{\textrm{init}}$ would be to fix the initial~$d_0$ and adjust the growth factor~$b$.

\section{Experimental evaluation}\label{sec:results}
In this section, we evaluate the performance of \thesbo on a 388D hyperparameter optimization task, a 124D design problem, and a collection of tasks that exhibit an active subspace.  
The \thesbo code is available at \url{https://github.com/LeoIV/BAxUS}.

\paragraph{The experimental setup.}
We benchmark against \turbo~\cite{eriksson2019scalable} with one and five trust regions, \saasbo~\cite{saasbo}, \alebo~\cite{alebo}, random search~\citep{bergstra2012random}, \cmaes~\cite{cma-es}, and \hesbo~\cite{nayebi2019a}, using the implementations provided by the respective authors with their settings, unless stated otherwise. 
For \cmaes, we use the \textsc{PyCma}~\cite{hansen2019pycma} implementation.
For \hesbo and \alebo, we use the \textsc{Ax} implementation~\cite{bakshy2018ae}.
To show the effect of different choices of $d$, we run \hesbo and \alebo with $d=10$ and $d=20$.
We observed that \alebo and \saasbo are constrained by their high runtime and memory consumption.
The available hardware allowed up to 100 function evaluations for \saasbo and 500 function evaluations for \alebo for each individual run.
Larger sampling budgets or higher target dimensions for \alebo resulted in out-of-memory errors. 
We point out that limited scalability was expected for these two methods, whereas the other methods scaled to considerably larger budgets, as required for scalable BO.
We initialize each optimizer, including \thesbo, with ten initial samples and \thesbo with $b=3$ and $m_D=1000$ and run 20 repeated trials.
Plots show the mean performance with one standard error. 

\paragraph{The benchmarks.}
We evaluate the selected algorithms on six benchmarks that differ considerably in their characteristics.
Following~\cite{wang2016bayesian}, we augment the \textsc{Branin2} and \textsc{Hartmann6} functions with additional dummy dimensions that have no influence on the function value.
We use the 388D \textsc{Svm} benchmark and the 124D soft-constraint version of the \textsc{Mopta08} benchmark proposed in~\citep{saasbo}.
We set a budget of~1000 evaluations for \textsc{Mopta08}, \textsc{Branin2}, and \textsc{Hartmann6} and of~2000 evaluations for the other benchmarks.
Moreover, we stopped for \textsc{Branin2} and \textsc{Hartmann6} when the simple regret dropped below~$.001$.
We show results on additional noise-free benchmarks in Appendices~\ref{app:lasso-dna} and \ref{app:mujoco}.
We also tested the algorithms on the 300D \textsc{Lasso-High} and the 1000D \textsc{Lasso-Hard} benchmarks from \lassobench~\cite{vsehic2021lassobench}.
These benchmarks have an effective dimensionality of~$5\%$ of the input dimensionality, i.e., the \textsc{Lasso-High} and \textsc{Lasso-Hard} benchmarks have 15 and 50 effective dimensions, respectively.
To study the robustness to observational noise, we also tested on noisy variants of \textsc{Lasso-Hard} and \textsc{Lasso-High}.

\subsection{Experimental results}\label{subsec:experimental-results}

\begin{figure}
\centering
\resizebox{\textwidth}{!}{\includeinkscape{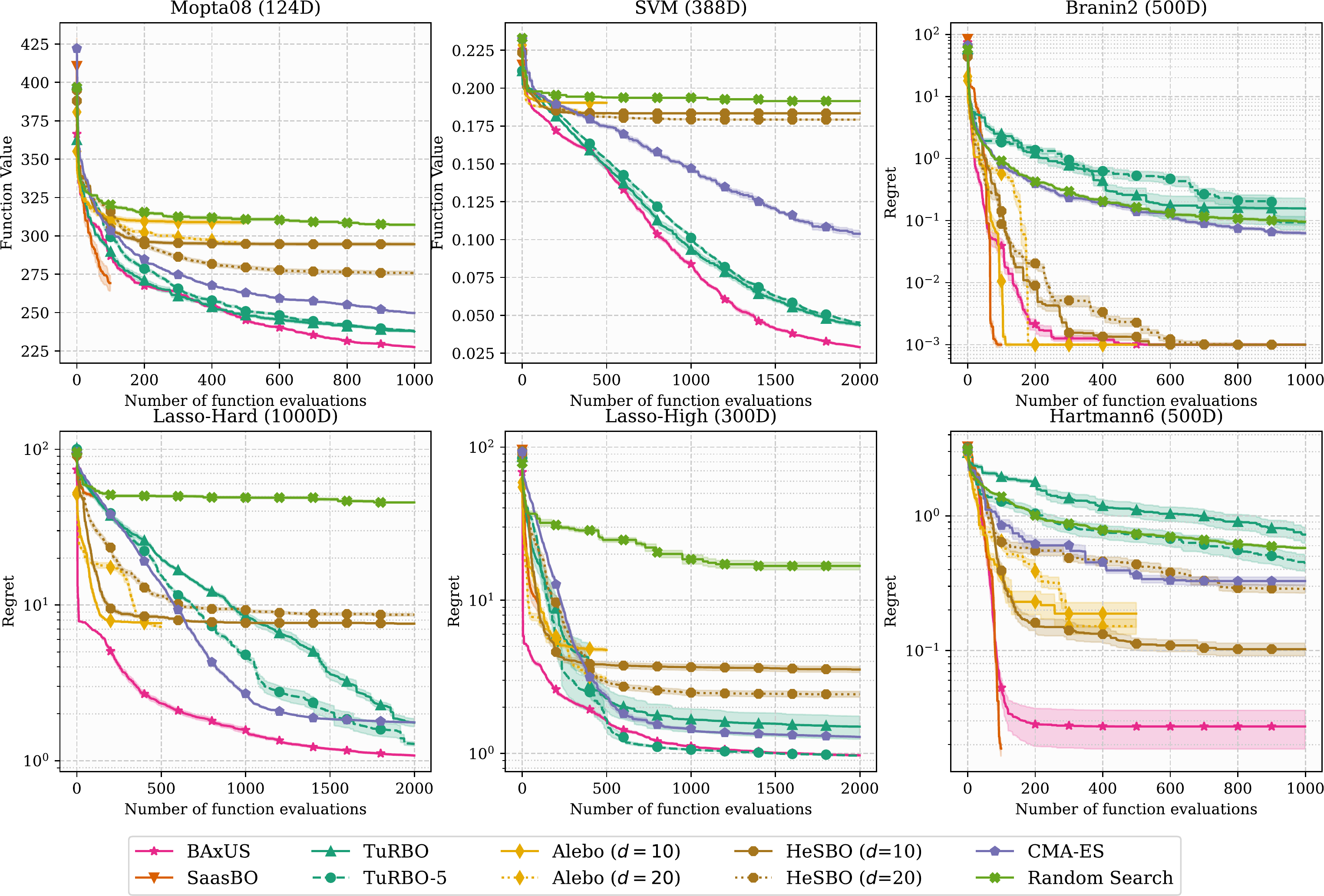}}
\caption{
Top row: \textbf{124D \textsc{Mopta08} (l)}: \thesbo obtains the best solutions, followed by \turbo and \cmaes. \textbf{388D \textsc{Svm} (c)}: \thesbo outperforms the other methods from the start. \textbf{500D \textsc{Branin} (r)}: \saasbo, \thesbo, \alebo, and \hesbo find an optimum; \saasbo and \alebo converge fastest.  Bottom row: \textbf{100D \textsc{Lasso-Hard} (l)} and \textbf{300D \textsc{Lasso-High} (c)}: \thesbo outperforms the baselines. \saasbo, \alebo, and \hesbo struggle. \textbf{500D \textsc{Hartmann6} (r)}: \saasbo performs best, closely followed by \thesbo. The other methods show only slow progress or stagnate.
}
\label{fig:lasso-hard-high}
\end{figure}

We begin with the six noise-free benchmarks.
Fig.~\ref{fig:lasso-hard-high} summarizes the performances.
On \textsc{Mopta08}, a 124D vehicle design problem, \saasbo initially makes slightly faster progress than \thesbo.
We suspect that this benchmark has high effective dimensionality, such that \thesbo first needs to adapt the target dimensionality to make further progress.
On the 388D \textsc{Svm} benchmark, \thesbo adapts to the appropriate target dimensionality where it can reach good function values faster than \turbo and \cmaes.
For this benchmark, \citet{saasbo} reported that \saasbo learned three active dimensions.
Yet, the fact that \alebo and \hesbo seem to stagnate after a few hundred evaluations, while \thesbo, \turbo, and \cmaes find better solutions, indicates that optimizing more of the 385 kernel length scales of the \textsc{Svm} benchmark allows for better solutions.
On the 500D \textsc{Hartmann6}, \saasbo performs best, closely followed by \thesbo. \alebo and \hesbo are competitive initially but converge to suboptimal solutions. \hesbo, \thesbo, \saasbo, \alebo all find excellent solutions on the \textsc{Branin} benchmark, with the latter algorithms converging faster.

Next, we examine the performances on the 1000D \textsc{Lasso-Hard} and the 300D \textsc{Lasso-High} that exhibit active subspaces, here without observational noise.
\thesbo achieves considerably better solutions than all state-of-the-art methods. 
We also note that \turbo and \cmaes perform better than \saasbo and \alebo.  
While one may expect \thesbo to outperform \turbo and \cmaes on these tasks with high input dimensions, it is surprising that \saasbo, \hesbo, and \alebo are not able to benefit from the present active subspace.
Here \thesbo's strategy to adaptively expand the nested subspace is superior.
Another crucial observation is that performances of \thesbo vary only slightly across runs.
Thus, \thesbo is robust despite the stochastic construction of the embedding.
Across the broad collection of benchmarks, \thesbo is the only method to consistently achieve high performance.

\paragraph{Noisy benchmarks.}
We evaluate the algorithms also for tasks with observational noise. 
Fig.~\ref{fig:lasso_hard_high_noisy} summarizes the results.
We observe that \thesbo achieves considerably better solutions for any number of observations than the competitors.
Moreover, we note that the performances of \saasbo, \cmaes, and \hesbo ($d=20$) degrade considerably on the \textsc{Lasso-High} task compared to the noise-free formulation of the task studied above. 
\thesbo' performance is equally strong as for the noise-free case and keeps making progress after 1000 observations.
\begin{figure}
    \centering
    \resizebox{ .85\linewidth}{!}{\includeinkscape{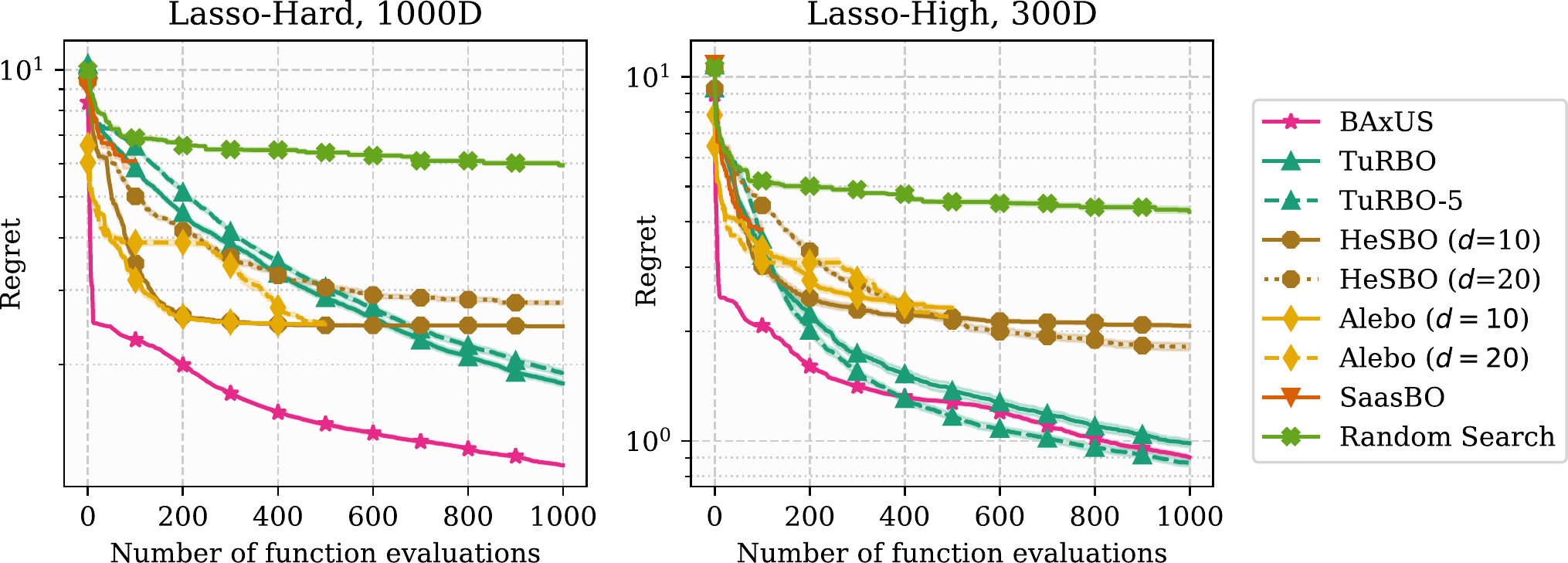}}
    \caption{\thesbo outperforms the SOTA and in particular proves to be robust to observational noise on the \textbf{1000D \textsc{Lasso-Hard} (l)} and the \textbf{300D \textsc{Lasso-High} (r)}. 
    Note that \cmaes performs considerably worse than on the noise-free versions of the benchmarks.
    }
    \label{fig:lasso_hard_high_noisy}
\end{figure}

\paragraph{\thesbo embedding ablation study.}
To investigate whether the proposed family of nested random subspaces contributes to the superior performance of \thesbo, we replaced the new embedding with a similar family of nested \hesbo embeddings. 
The results show that the proposed embedding provides a significant performance gain.
Due to space constraints, the results were moved to Appendix~\ref{app:ablation_axus}.

\section{Discussion}\label{sec:discussion}
\Acl{HDBO} is aspiring to unlock impactful applications broadly in science and industry.
However, state-of-the-art methods suffer from limited scalability or, in some cases, require practitioners to `guess' certain hyperparameters that critically impact the performance.
This paper proposes \thesbo that works out-of-the-box and achieves considerably better performance for high-dimensional problems, as the comprehensive evaluation shows.
A key idea is to scale up the dimensionality of the target subspace that the algorithm optimizes over.
We apply a simple strategy that we find to work well across the board. 
However, we expect substantial headroom in tailoring this strategy to specific applications, either using domain expertise or a more sophisticated data-driven approach that, for example, learns a suitable target space.
Moreover, future work will explore extending \thesbo to structured domains, particularly the combinatorial spaces common in materials sciences and drug discovery.

\paragraph{Societal impact.}
Bayesian optimization has recently become a popular tool for tasks in drug discovery~\cite{negoescu2011the}, chemical engineering~\cite{lobato2017parallel,shields2021bayesian,hase2018phoenics,schweidtmann2018machine,burger2020mobile}, materials science~\cite{UENO201618,Frazier2016,packwood2017bayesian,souza2019deepfreak,herbol2018efficient,hase2021gryffin}, aerospace engineering~\cite{lukaczyk2014active,baptista2018bayesian,lam2018advances}, robotics~\cite{lizotte2007automatic,calandra2016bayesian,rai2018bayesian,tuprints5878,mayr2022skill},  and many more. 
This speaks to the progress that Bayesian optimization has made in becoming a robust and reliable `off-the-shelf solver.' 
However, this promise is not yet fulfilled for the newer field of high-dimensional Bayesian optimization that allows optimization over hundreds of `tunable levers.'
The abovementioned applications benefit from incorporating more such levers in the optimization: it allows for more detailed modeling of an aerospace design or a more granular control of a chemical reaction, to give some examples.
The evaluation shows that the performance of state-of-the-art methods degenerates drastically for such high dimensions if the application does not meet specific requirements. 
Adding insult to the injury, such requirements as the dimensionality of an active subspace cannot be determined beforehand.

The proposed algorithm achieves a robust performance over a broad collection of tasks and thus will become a `goto' optimizer for practitioners in other fields. 
Therefore, we released the \thesbo code.

\begin{ack}
Luigi Nardi was supported in part by affiliate members and other supporters of the Stanford DAWN project — Ant Financial, Facebook, Google, Intel, Microsoft, NEC, SAP, Teradata, and VMware.
Leonard Papenmeier and Luigi Nardi were partially supported by the Wallenberg AI, Autonomous Systems and Software Program (WASP) funded by the Knut and Alice Wallenberg Foundation. 
Luigi Nardi was partially supported by the Wallenberg Launch Pad (WALP) grant Dnr 2021.0348. 
The computations were also enabled by resources provided by the Swedish National Infrastructure for Computing (SNIC) at LUNARC, partially funded by the Swedish Research Council through grant agreement no. 2018-05973.
We would like to thank Erik Hellsten from Lund University and Eddy de Weerd from the University of Twente for valuable discussions.
\end{ack}

\newpage

\bibliography{references}

\clearpage
\section*{Checklist}

\begin{enumerate}
\item For all authors...
\begin{enumerate}
  \item Do the main claims made in the abstract and introduction accurately reflect the paper's contributions and scope?
    \answerYes{}
  \item Did you describe the limitations of your work?
    \answerYes{} See Section~\ref{sec:discussion}.
  \item Did you discuss any potential negative societal impacts of your work?
    \answerYes{} See Section~\ref{sec:discussion}.
  \item Have you read the ethics review guidelines and ensured that your paper conforms to them?
    \answerYes{}
\end{enumerate}

\item If you are including theoretical results...
\begin{enumerate}
  \item Did you state the full set of assumptions of all theoretical results?
    \answerYes{}
        \item Did you include complete proofs of all theoretical results? 
    \answerYes{} See Appendix~\ref{sec:proofs_embedding}.
\end{enumerate}

\item If you ran experiments...
\begin{enumerate}
  \item Did you include the code, data, and instructions needed to reproduce the main experimental results (either in the supplemental material or as a URL)?
    \answerYes{}
  \item Did you specify all the training details (e.g., data splits, hyperparameters, how they were chosen)?
    \answerYes{}
        \item Did you report error bars (e.g., with respect to the random seed after running experiments multiple times)?
    \answerYes{}
        \item Did you include the total amount of compute and the type of resources used (e.g., type of GPUs, internal cluster, or cloud provider)? 
    \answerYes{} See Appendix~\ref{append:implementation_details}.
\end{enumerate}

\item If you are using existing assets (e.g., code, data, models) or curating/releasing new assets...
\begin{enumerate}
  \item If your work uses existing assets, did you cite the creators?
    \answerYes{} See Appendix~\ref{append:implementation_details}.
  \item Did you mention the license of the assets?
    \answerYes{} See Appendix~\ref{append:implementation_details}.
  \item Did you include any new assets either in the supplemental material or as a URL?
    \answerYes{}
  \item Did you discuss whether and how consent was obtained from people whose data you're using/curating?
    \answerNA{}
  \item Did you discuss whether the data you are using/curating contains personally identifiable information or offensive content?
    \answerNA{}
\end{enumerate}

\item If you used crowdsourcing or conducted research with human subjects...
\begin{enumerate}
  \item Did you include the full text of instructions given to participants and screenshots, if applicable?
    \answerNA{}
  \item Did you describe any potential participant risks, with links to Institutional Review Board (IRB) approvals, if applicable?
    \answerNA{}
  \item Did you include the estimated hourly wage paid to participants and the total amount spent on participant compensation?
    \answerNA{}
\end{enumerate}

\end{enumerate}

\newpage
\appendix

\section{Theoretical foundation for the \thesbo embedding}\label{sec:proofs_embedding}
For convenience, we re-state Definition~\ref{def:count-sketch-type} and Definition~\ref{def:success} from Section~\ref{sec:algorithm}.

\newtheorem*{definition1}{Definition~\ref{def:count-sketch-type}}
\begin{definition1}[\textbf{\Sparseembedding matrix}]
    A matrix $S\in \{0,\pm 1\}^{d\times D}$ is a \emph{\sparseembedding matrix} if and only if each column in $S$ has exactly one non-zero entry~\cite{woodruff2014sketching}.
\end{definition1}

\newtheorem*{definition2}{Definition~\ref{def:success}}
\begin{definition2}[\textbf{Success of a \sparseembedding}]
    A success of a random \sparseembedding is the event $Y^*=$ ``All $d_e$ active input dimensions are mapped to distinct target dimensions.''
\end{definition2}

We introduce the following two definitions.

\begin{definition}[\textbf{Optima-preserving \sparseembedding}]
    A \sparseembedding matrix is \emph{optima-preserving} if each target dimension (i.e., each column in $S$) contains at most one active input dimension.
\end{definition}

\begin{definition}[\textbf{Sparse function / function with an active subspace}]
    \label{def:function-with-active-subspace}
    Let $\mathcal{X} = [-1,1]^D$.
    A function $f:\mathcal{X}\rightarrow \mathbb{R}$ has an \emph{active subspace} (or \emph{effective subspace}~\cite{wang2016bayesian}), if there exist a subspace (i.e., a space $\mathcal{Z}\subseteq \mathbb{R}^{d_e}$, with $d_e\leq D$ where $d_e\in \mathbb{N}_{++}$ is the effective dimensionality and $\mathbb{N}_{++} = \mathbb{N}\setminus \{0\}$) and a projection matrix $S^\T \in \mathbb{R}^{D\times d_e}$, such that for any $\bm{x}\in \mathcal{X}$ there exists a $\bm{z}\in \mathcal{Z}$ so that $f(\bm{x}) = f(S^\T\bm{z})$ and $d_e$ is the smallest integer with this property.
    The function is called \textit{sparse} if it has an active subspace and $S^\T$ is a \sparseembedding matrix and $\mathcal{Z}=[-1,1]^{d_e}$.
\end{definition}

\subsection{Proof of Theorem~\ref{lemma:ft_embedding}}\label{proof:lemma_ft_embedding}
We prove the worst-case success probability for the \thesbo embedding.

\newtheorem*{lemma1}{Theorem~\ref{lemma:ft_embedding}}
\begin{lemma1}[\textbf{Worst-case success probability of the \thesbo embedding}]
    Let $D$ be the input dimensionality and $d\geq d_e$ the dimensionality of the embedding.
Let $\beta_{\textrm{small} } = \left \lfloor \frac{D}{d} \right \rfloor$ and $\beta_{\textrm{large} } = \left \lceil \frac{D}{d} \right \rceil$ be the small and large bin sizes.
Then the probability of $Y^*$ (see Definition~\ref{def:success}) for the \thesbo embedding is
\begin{equation}
    p_B(Y^*; D,d,d_e) = \frac{\sum_{i=0}^{d_e}\binom{d(1+\beta_{\textrm{small} })-D}{i}\binom{D-d\beta_{\textrm{small} }}{d_e-i}\beta_{\textrm{small} }^i\beta_{\textrm{large} }^{d_e-i}}{\binom{D}{d_e}}.\label{eq:binning}
\end{equation}
\end{lemma1}
\begin{proof}

    The assignment of input dimensions to target dimensions and the signs of the input dimensions fully define the \thesbo embedding.
    Note that the signs do not affect $p_B(Y^*; D,d,d_e)$ because they only correspond to ``flipping'' the input dimension in the target space, and our construction ensures that the value ranges are symmetric to the origin.

    An assignment is optima-preserving if and only if it is possible to find a point in $\mathcal{Y}$ that maps to an optimum in $\mathcal{X}$ for any $f$.
    The ``only if'' is true because $f$ is assumed to be sparse with an active subspace with $d_e$ active dimensions.
    This means that the optima in $\mathcal{X}$ only change their function values along the $d_e$ active dimensions.
    Suppose it is possible to find a point in $\mathcal{Y}$ that maps to an arbitrary optimum in $\mathcal{X}$.
    In that case, the assignment is optima-preserving because it can individually adjust all the $d_e$ active dimensions in $\mathcal{X}$.
    However, this generally requires each active input dimension to be mapped to a distinct target dimension (note that we require being able to represent the optimum for \emph{any} $f$).
    Otherwise, there would be at least two active input dimensions that cannot be changed independently.
    Therefore, the probability of $Y^*$ equals the probability of an optima-preserving assignment.

    As all assignments are equally likely under the construction, the probability of an assignment being optima-preserving is equal to the number of possible optima-preserving assignments divided by the total number of assignments.
    There are $\binom{D}{d_e}$ ways of distributing the $d_e$ active dimensions across the $D$ positions, giving the denominator in Eq.~\eqref{eq:binning}.

    Let us first assume that $\beta_{\textrm{small} } = \beta_{\textrm{large} } = \beta$, i.e., all target dimensions have the same number of input dimensions and $d$ divides $D$.
    We refer to this case as the \emph{balanced case}.
    There are $\binom{d}{d_e}$ ways of distributing the $d_e$ active dimensions across the $d$ different target dimensions.
    Given one active dimension, there are $\beta$ ways in which this dimension can map to the target dimension.
    Therefore, for the balanced case, the worst-case success probability is given by
    \begin{align}
        p_B(Y^*; D,d,d_e)=\frac{\beta^{d_e}\binom{d}{d_e}}{\binom{D}{d_e}}.
        \label{eq:d_divides_D_lemma}
    \end{align}
    Next, we generalize Eq.~\eqref{eq:d_divides_D_lemma} for cases where $d$ does not divide $D$.
    We refer to this case as the \emph{near-balanced case}.
    In that case, there are two bin sizes: $\beta_{\textrm{small} }$ and $\beta_{\textrm{large} }$ with $\beta_{\textrm{large} } = \beta_{\textrm{small} }+1$.
    There are $d\beta_{\textrm{large} }-D$ small bins (i.e., bins with bin size $\beta_{\textrm{small} }$) and $D-d\beta_{\textrm{small} }$ large bins:
    $D-d\beta_{\textrm{small}}$ gives the number of input dimensions that would not be covered if all bins were small.
    Since $\beta_{\textrm{small} }$ and $\beta_{\textrm{large}}$ differ by $1$, this also gives the number of bins that have to be large.
    Conversely, if we only had large bins, we would cover $d\beta_{\textrm{large} }-D$ too many input dimensions.
    Therefore, we need $D-d\beta_{\textrm{small}}$ large and $d\beta_{\textrm{large} }-D$ small bins. %

    We consider all ways of distributing the $d_e$ active dimensions across the the $d\beta_{\textrm{large} }-D$ small and $D-d\beta_{\textrm{small}}$ large bins so that there is at most one active dimension in each bin. 
    Recall that this number gives the numerator in Eq.~\eqref{eq:binning}.
    For a conflict-free assignment, if~$i$ active dimensions are mapped to small bins, then~$d_e-i$ active dimensions must be assigned to large bins. 
    There are $\binom{d(1+\beta_{\textrm{small} })-D}{i}\binom{D-d\beta_{\textrm{small} }}{d_e-i}$ such assignments. Here we use that~$1+\beta_{\textrm{small}} = \beta_{\textrm{large}}$ holds for the near-balanced case. 
    Recall that each small bin has~$\beta_{\textrm{small} }$ locations and that each large bin has~$\beta_{\textrm{large}}$ locations that an active dimension can be assigned to.
    Because~$0 \leq i \leq d_e$ by construction, the number of assignments that result in an optima-preserving embedding is
    \begin{equation*}
      \sum_{i=0}^{d_e}\binom{d(1+\beta_{\textrm{small} })-D}{i}\binom{D-d\beta_{\textrm{small} }}{d_e-i}\beta_{\textrm{small} }^i\beta_{\textrm{large} }^{d_e-i}.  
    \end{equation*}
    Note that we leverage the facts  $\binom{0}{0}=1$, $\binom{0}{x} = 0$ for all $x \geq 1$, $\binom{y}{x}=0$ if $x>y\geq 0$, and $\binom{x}{0}=1$ for all $x$, thus the sum is well defined.
    Recall that we already showed that the denominator is $\binom{D}{d_e}$.
    Therefore, Eq.~\eqref{eq:binning} gives the success probability in the near-balanced case.

    It is easy to see that Eq.~\eqref{eq:binning} is equivalent to the near-balanced formulation in Eq.~\eqref{eq:d_divides_D_lemma} when $d$ divides $D$.
    When $d$ divides $D$, $\beta_{\textrm{small} } = \beta_{\textrm{large} } = \beta$, $d(1+\beta_{\textrm{small} })-D=d$, and $D-d\beta_{\textrm{small} }=0$.
    Therefore, the worst-case success probability for the near-balanced case is given by
    \begin{align*}
        p_B(Y^*; D,d,d_e)&=\frac{\sum_{i=0}^{d_e}\binom{d(1+\beta_{\textrm{small} })-D}{i}\binom{D-d\beta_{\textrm{small} }  }{d_e-i}\beta_{\textrm{small} }^i\beta_{\textrm{large} }^{d_e-i}}{\binom{D}{d_e}}\\
        &\overset{\beta_{\textrm{small} }=\beta_{\textrm{large} }}{=} \frac{\sum_{i=0}^{d_e}\binom{d}{i}\binom{0}{d_e-i}\beta^i\beta^{d_e-i}}{\binom{D}{d_e}}\\
        &= \frac{\beta^{d_e}\sum_{i=0}^{d_e}\binom{d}{i}\binom{0}{d_e-i}}{\binom{D}{d_e}}= \frac{\beta^{d_e}\binom{d}{d_e}}{\binom{D}{d_e}}
    \end{align*}
    where the last equality is true because the sum is zero unless $i=d_e$.\\
\end{proof}

\subsection{Proof of Corollary~\ref{corr:optimality}}\label{proof:optimality}
We prove the optimality of the \thesbo embedding in terms of the worst-case success probability.
\newtheorem*{corr2}{Corollary~\ref{corr:optimality}}
\begin{corr2}
    With the same input, target, and effective dimensionalities ($D$, $d$, and $d_e$), no \sparseembedding has a higher worst-case success probability than the \thesbo embedding.

\end{corr2}

\begin{proof}
    By Definition~\ref{def:count-sketch-type}, an embedding matrix $S\in \{0,\pm 1\}^{D\times d}$ is sparse if each row in $S$ has exactly one non-zero entry.
    Such an embedding can always be interpreted as disjoint sets of signed input dimensions assigned to different target dimensions: For the $n$-th input dimension, find the column with the non-zero.
    The respective column gives the target dimension;
    the entry in the matrix itself gives the sign.
    Conversely, each target dimension has a set of contributing input dimensions, and we call the set of input dimensions mapping to a target dimension a ``bin''.
    The sign of the input dimensions does not influence the success probability as it does not influence the ability of an embedding to contain the optimum.

    We will prove that the \thesbo embedding is optimal, i.e., every other \sparseembedding has a worst-case success probability that is lower or equal.
    We start by giving the worst-case success probability for arbitrary bin sizes.

    Let $\beta_n$ be the bin size of the $n$-th bin.
    By Definition~\ref{def:success}, a success is guaranteed if each bin contains at most one active input dimension.
    Therefore, the worst-case success probability for arbitrary bin sizes has to consider the number of cases where each bin contains at most one active input dimension and the number of bins containing one active input dimension is equal to the number of active input dimensions $d_e$.
    In a bin of size $\beta_i$, the active input dimension can lie in $\beta_i$ different locations.
    Bins not containing an active dimension do not contribute to the worst-case success probability.

    We suppose w.l.o.g.\ that $D\geq d$. Thus, every target dimension has at least one input dimension.
    For each $n$ from $1$ to $d$, let the value $i_n$ indicate whether the $n$-th bin (or target dimension) contains an active dimension ($i_n=1$) or not ($i_n=0$).
    The indicator variable $\mathds{1}_{(\sum_{n=1}^d i_n) = d_e}$ ensures that only cases where exactly $d_e$ bins contain an active input dimension are counted.
    Note that $\sum_{i_1=0}^{1}\sum_{i_2=0}^{1}\ldots \sum_{i_{d}=0}^{1} \mathds{1}_{(\sum_{n=1}^d i_n) = d_e} = \binom{d}{d_e}$.
    For each case where the $d_e$ active dimensions are assigned to $d_e$ out of $d$ disjoint bins, the term $\prod_{n=1}^{d}\beta_n^{i_n}$ accounts for the locations in which the active dimension can lie in the $n$-th bin.
    Other cases do not contribute to the worst-case success probability.
    The exponent ensures that only bins containing an active dimension contribute to the denominator.

    Then the worst-case success probability for arbitrary bin sizes is given by
    \begin{equation}
        p_{\textrm{general} }(Y^*;D,d,d_e)=\frac{\overbrace{\sum_{i_1=0}^{1}\sum_{i_2=0}^{1}\ldots \sum_{i_{d}=0}^{1} \mathds{1}_{(\sum_{n=1}^d i_n) = d_e}}^{=\binom{d}{d_e}} {\displaystyle \prod_{n=1}^{d}\beta_n^{i_n}}}{\binom{D}{d_e}} \label{eq:general_failure_probability},
    \end{equation}
    with $\beta_n > 0$, $\sum_{n=1}^d \beta_n = D$, and $d\geq d_e$:

    As in Theorem~\ref{lemma:ft_embedding}, the denominator of Eq.~\eqref{eq:general_failure_probability} gives all ways of assigning $d_e$ active dimensions to $D$ input dimensions.

    We now prove that any \sparseembedding has a worst-case success probability that is less or equal to the worst-case success probability of the \thesbo embedding.

    Let $\beta_{\textrm{small} }$, $\beta_{\textrm{large} }$, and $p_B(Y^*; D,d,d_e)$ as in Theorem~\ref{lemma:ft_embedding}.
    Then,

    \begin{equation}
        p_{\textrm{general} }(Y^*;D,d,d_e) \leq  p_B(Y^*; D,d,d_e) = \frac{\overbrace{\sum_{i=0}^{d_e}\binom{d(1+\beta_{\textrm{small} })-D}{i}\binom{D-d\beta_{\textrm{small} }  }{d_e-i}}^{=\binom{d}{d_e}}\beta_{\textrm{small} }^i\beta_{\textrm{large} }^{d_e-i}}{\binom{D}{d_e}}.  \label{eq:optimality_proof_line}
    \end{equation}

    We refer the reader to the proof of Theorem~\ref{lemma:ft_embedding} for an explanation of the binomial coefficients.
    The fact that $\sum_{i=0}^{d_e}\binom{d(1+\beta_{\textrm{small}})-D}{i}\binom{D-d\beta_{\textrm{small}}}{d_e-i} = \binom{d}{d_e}$ can be seen by noting that $(d(1+\beta_{\textrm{small}})-D) + (D-d\beta_{\textrm{small}}) = d$ and applying Vandermonde's convolution~\cite{graham1989concrete}.

    We will now prove that if $d$ divides $D$, then the product in the numerator of Eq.~\eqref{eq:general_failure_probability} is maximized if all the factors are the same, i.e., $\beta=\frac{D}{d}$.
    We will then show that if $d$ does not divide $D$, the integer-solution of maximal value is attained for $\beta_{\textrm{large} }-\beta_{\textrm{small} }=1$.

    \paragraph{First case ($d$ divides $D$)}
    We now show that the following holds for the term $\prod_{n=1}^{d}\beta_n^{i_n}$ in the numerator of Eq.~\eqref{eq:general_failure_probability}: $\prod_{n=1}^{d}\beta_n^{i_n}\leq \beta^{d_e}$.
    The numerator in Eq.~\eqref{eq:general_failure_probability} can also be written as $e_{d_e}(\beta_1,\ldots,\beta_d)$ where
    \begin{align*}
        e_{d_e}(\beta_1,\ldots,\beta_d)&=\sum_{i_1<i_2<\ldots<i_{d_e}}\beta_{i_1}\beta_{i_2}\ldots\beta_{i_{d_e}}\\
        &=\sum_{i_1=0}^{1}\sum_{i_2=0}^{1}\ldots \sum_{i_{d}=0}^{1} \mathds{1}_{(\sum_{n=1}^d i_n) = d_e}{\displaystyle \prod_{n=1}^{d}\beta_n^{i_n}}
    \end{align*}
    is the $d_e$-th elementary symmetric function of $\beta_1,\ldots,\beta_d$~\cite{beckenbach2012inequalities}.
    Maclaurin's inequality~\cite{beckenbach2012inequalities} states that
    \begin{align}
        \frac{e_1(\beta_1,\ldots,\beta_d)}{\binom{d}{1}}\geq \sqrt{\frac{e_2(\beta_1,\ldots,\beta_d)}{\binom{d}{2}}} \geq \ldots \geq \sqrt[d_e]{\frac{e_{d_e}(\beta_1,\ldots,\beta_d)}{\binom{d}{d_e}}} \geq \ldots \geq \sqrt[d]{\frac{e_d(\beta_1,\ldots,\beta_d)}{\binom{d}{d}}}. \label{eq:maclaurin}
    \end{align}
    In particular,
    \begin{align}
        \frac{e_1(\beta_1,\ldots,\beta_d)}{\binom{d}{1}} = \frac{\sum_{i=1}^d\beta_i}{d}=\frac{D}{d}=\beta \label{eq:maclaurin_1}
    \end{align}
    holds.
    Taking Eq.~\eqref{eq:maclaurin} and Eq.~\eqref{eq:maclaurin_1} to the power $d_e$ and multiplying by $\binom{d}{d_e}$, we obtain
    \begin{equation}
        \beta^{d_e}\binom{d}{d_e}\geq e_{d_e}(\beta_1,\ldots,\beta_d) = \sum_{i_1=0}^{1}\sum_{i_2=0}^{1}\ldots \sum_{i_{d}=0}^{1} \mathds{1}_{(\sum_{n=1}^d i_n) = d_e}{\displaystyle \prod_{n=1}^{d}\beta_n^{i_n}} \label{eq:elementary_symmetric_function},
    \end{equation}
    with equality if and only if $\beta_i=\beta_j$ for $i,j\in \{1,\ldots,n\}$~\cite{beckenbach2012inequalities}.
    Therefore, the product in the numerator of Eq.~\eqref{eq:general_failure_probability} is maximized if all factors are equal.

    \paragraph{Second case ($d$ does not divide $D$)}
    However, if $d$ does not divide $D$, then $\beta$ is no integer which is not feasible in our setting.
    The $d_e$-th elementary symmetric function $e_{d_e}(\bm{\beta})$ (see Eq.~\eqref{eq:elementary_symmetric_function}) is known to be \emph{Schur-concave} if $\beta_i\geq 0$ holds for all $i$~\cite{rovencta2012note}.
    This condition is met by~$\beta$.
    We use the following definition of~\cite{rovencta2012note}: A function $f: \mathbb{R}^d \rightarrow \mathbb{R}$ is called Schur-concave if $\bm{\gamma} \prec \bm{\beta}$ implies $f(\bm{\gamma})\geq f(\bm{\beta})$.
    Here, $\bm{\gamma} \prec  \bm{\beta}$ means that $\bm{\beta}$ \emph{majorizes} $\bm{\gamma}$, i.e.,
    \[
        \sum_{i=1}^k \bm{\gamma}^{\downarrow}_i \leq \sum_{i=1}^k \bm{\beta}^{\downarrow}_i\quad \text{for all $k\in \{1,\ldots,d\}$},\quad \text{and}\quad \sum_{i=1}^d \bm{\gamma}_i = \sum_{i=1}^d \bm{\beta}_i,
    \]
    where $\bm{\gamma}^{\downarrow}$ and $\bm{\beta}^{\downarrow}$ are the vectors of all elements in $\bm{\gamma}$ and $\bm{\beta}$ in descending order~\cite{rovencta2012note}.

    We now show that there is no integer solution $\bm{\gamma}$ such that there is a near-balanced solution that majorizes $\bm{\gamma}$.

    For some near-balanced assignment $\bm{\beta}$ of small and large bins to the $d$ target dimensions, consider the vector
    \[
        \bm{\beta}^{\downarrow} = \left ( \underbrace{\overbrace{\left \lceil \frac{D}{d} \right \rceil}^{=\beta_{\textrm{large} }}, \ldots, \left \lceil \frac{D}{d} \right \rceil}_{D-d\beta_{\textrm{small} }\text{ many}}, \underbrace{\overbrace{\left \lfloor \frac{D}{d} \right \rfloor}^{=\beta_{\textrm{small} }}, \ldots, \left \lfloor \frac{D}{d} \right \rfloor}_{d\beta_{\rm large}-D\text{ many}} \right )
    \]
    of bin sizes in decreasing order.
    For any other \thesbo embedding given by some permutation $\bm{\beta}'$ of $\bm{\beta}$, it holds that $\bm{\beta}'^{\downarrow} = \bm{\beta}^{\downarrow}$.
    Note that for any assignment $\bm{\gamma} = \{\bm{\gamma}_1, \ldots, \bm{\gamma}_d\}$ of bin sizes over the $d$ target dimensions, it has to hold that
    \[
        \sum_{i=1}^d \bm{\gamma}_i = D;\, \bm{\gamma}_i \geq 0,\, i \in \{1,\ldots,d\};\, \bm{\gamma}\in \mathbb{N}_0^{d}.
    \]
    By assumption, since we are in the near-balanced case, $\beta_{\textrm{small} }=\beta_{\textrm{large} }-1$.

    Assume there exists an assignment of bin sizes $\bm{\gamma}$ that is not a permutation of $\bm{\beta}$ such that $\bm{\gamma} \prec \bm{\beta}$, i.e.,
    \begin{align}
        &\sum_{i=1}^k \bm{\gamma}^{\downarrow}_i \leq \sum_{i=1}^k \bm{\beta}^{\downarrow}_i\quad \text{for all $k\in \{1,\ldots,d\}$},\\
        \intertext{and}
        &\sum_{i=1}^d \bm{\gamma}_i = \sum_{i=1}^d \bm{\beta}_i,
        \intertext{and}
        &\exists j: \bm{\gamma}^{\downarrow}_j < \bm{\beta}^{\downarrow}_j.
    \end{align}
    Let $\bm{\kappa}$ denote the (non-empty) set of such indices.
    Because the elements of $\bm{\beta}$ and $\bm{\gamma}$ both sum up to $D$, it has to hold for all $\kappa\in \bm{\kappa}$ that
    \begin{align}
        \sum_{i=1, i\neq \kappa}^d \bm{\gamma}^{\downarrow}_i > \sum_{i=1, i\neq \kappa}^d \bm{\beta}^{\downarrow}_i. \label{eq:schur-remaining-sum}
    \end{align}

    Remember that $\bm{\beta}$ only contains elements of sizes $\beta_{\rm small}$ and $\beta_{\rm large}$ with $\beta_{\rm large} = \beta_{\rm small}-1$.
    Then, Eq.~\eqref{eq:schur-remaining-sum} can only hold if either 1) $\bm{\gamma}$ contains more elements of size $\beta_{\rm large}$ than $\bm{\beta}$ or 2) if it contains at least one element that is larger than $\beta_{\rm large}$, the largest element in $\bm{\beta}$.

    Both cases lead to a contradiction.
    In the first case,
    \[
        \sum_{i=1}^{D-d\beta_{\rm small}+1}\bm{\gamma}^{\downarrow}_i > \sum_{i=1}^{D-d\beta_{\rm small}+1}\bm{\beta}^{\downarrow}_i \Rightarrow  \bm{\gamma} \not \prec \bm{\beta}
    \]
    since at least the first $D-d\beta_{\rm small}+1$ elements of $\bm{\gamma}^{\downarrow}$ are $\left \lceil \frac{D}{d} \right \rceil$ but only the first $D-d\beta_{\rm small}$ elements of $\bm{\beta}^{\downarrow}$ are $\left \lceil \frac{D}{d} \right \rceil$ and the $D-d\beta_{\rm small}+1$-th element of $\bm{\beta}^{\downarrow}$ is $\left \lceil \frac{D}{d}\right \rceil-1$.

    In the second case, $\bm{\gamma} \not \prec \bm{\beta}$ because $\bm{\gamma}^{\downarrow}_1 > \bm{\beta}^{\downarrow}_1$.
    It follows that no such $\bm{\gamma}$ exists.
    Therefore, the \thesbo embedding has a maximum worst-case success probability among sparse embeddings.
\end{proof}

\subsection{Proof of Corollary~\ref{corr:equivalence}}\label{proof:equivalence}

\newtheorem*{corr1}{Corollary~\ref{corr:equivalence}}
\begin{corr1}
    For $D\rightarrow \infty$, the worst-case success probability of the \thesbo embedding is
\begin{equation}
    \lim_{D\rightarrow \infty}p_B(Y^*;D,d,d_e)=\frac{d!}{(d-d_e)!d^{d_e}},
\end{equation}
and hence matches \hesbo's worst-case success probability $p_H(Y^*;d,d_e)$.
\end{corr1}
\begin{proof}
By Corollary~\ref{corr:equivalence}, the following holds for arbitrary $1\leq d_e \leq d \leq D$ where $d,d_e,D\in \mathbb{N}_{++}$:
\[
    \frac{d!}{(d-d_e)!d^{d_e}} \leq p_B(Y^*; D, d ,d_e),
\]
because $\frac{d!}{(d-d_e)!d^{d_e}}$ is \hesbo's worst-case success probability and hence less or equal to the worst-case success probability of \thesbo.

Furthermore, by the proof of Corollary~\ref{corr:equivalence},
\[
    p_B(Y^*; D, d ,d_e) \leq \frac{\beta^{d_e}\binom{d}{d_e}}{\binom{D}{d_e}},
\]
because $\frac{\beta^{d_e}\binom{d}{d_e}}{\binom{D}{d_e}}$ is larger or equal the worst-case success probability of any \sparseembedding, among which \thesbo is the embedding with maximum worst-case success probability and $\frac{\beta^{d_e}\binom{d}{d_e}}{\binom{D}{d_e}} = p_B(Y^*; D, d ,d_e)$ if and only if $\beta_{\rm small}=\beta_{\rm large}$, i.e., $d$ divides $D$.

In summary, we have
\[
    \frac{d!}{(d-d_e)!d^{d_e}} \leq p_B(Y^*; D, d ,d_e) \leq \frac{\beta^{d_e}\binom{d}{d_e}}{\binom{D}{d_e}}.
\]

We now show that, for fixed $d$ and $d_e$, the sequences $\frac{d!}{(d-d_e)!d^{d_e}}$ and $\frac{\beta^{d_e}\binom{d}{d_e}}{\binom{D}{d_e}}$ converge to the same point as $D\rightarrow \infty$.
We consider $\lim_{\beta \rightarrow \infty}$, which is equivalent to $\lim_{D\rightarrow \infty}$ as $\beta = \frac{D}{d}$ and $d$ is fixed.
Note, that we can consider $\beta = \frac{D}{d}$ even though it is not a valid success probability when $d$ does not divide $D$, since we are only interested in bounding the true success probability.
Then,
\begin{align}
    \lim_{\beta \rightarrow \infty} \frac{\beta^{d_e}\binom{d}{d_e}}{\binom{D}{d_e}}
    &= \lim_{\beta \rightarrow \infty} \frac{\beta^{d_e}d!(D-d_e)!}{D!(d-d_e)!} \\
    &=\lim_{\beta \rightarrow \infty} \frac{\beta^{d_e}d!(\beta d-d_e)!}{(\beta d)!(d-d_e)!}\\
    &= \lim_{\beta \rightarrow \infty} \beta^{d_e}\frac{d!}{(d-d_e)!}\frac{(\beta d - d_e)!}{(\beta d)!}
    \intertext{Applying Stirling's approximation~\cite{graham1989concrete} to the numerator and the denominator of the last factor, we obtain}
    &= \lim_{\beta \rightarrow \infty} \beta^{d_e} \frac{d!}{(d-d_e)!}  \frac{\sqrt{2\pi(\beta d-d_e)}\left ( \frac{\beta d -d_e}{e} \right )^{\beta d-d_e}}{\sqrt{2\pi \beta d}\left ( \frac{\beta d}{e}\right )^{\beta d}}\frac{r(\beta d-d_e)}{r(\beta d)}\\
    &= \lim_{\beta \rightarrow \infty} \beta^{d_e} \frac{d!}{(d-d_e)!} \sqrt{\frac{\beta d-d_e}{\beta d}}e^{d_e}\frac{(\beta d-d_e)^{\beta d-d_e}}{(\beta d)^{\beta d}}\frac{r(\beta d-d_e)}{r(\beta d)}\\
    &= \lim_{\beta \rightarrow \infty} \beta^{d_e} \frac{d!}{(d-d_e)!} \sqrt{\frac{\beta d -d_e}{\beta d}}e^{d_e} \left ( \frac{\beta d-d_e}{\beta d} \right )^{\beta d} \frac{1}{(\beta d-d_e)^{d_e}}\frac{r(\beta d-d_e)}{r(\beta d)}\\
    &= \lim_{\beta \rightarrow \infty}  \frac{d!}{(d-d_e)!} \sqrt{\frac{\beta d-d_e}{\beta d}}e^{d_e} \left ( \frac{\beta d-d_e}{\beta d} \right )^{\beta d} \frac{\beta^{d_e}}{(\beta d-d_e)^{d_e}} \frac{r(\beta d-d_e)}{r(\beta d)}\\
    &= \lim_{\beta \rightarrow \infty}  \frac{d!}{(d-d_e)!} \underbrace{\sqrt{\frac{\beta d-d_e}{\beta d}}}_{\rightarrow 1}e^{d_e} \underbrace{\left ( \frac{\beta d-d_e}{\beta d} \right )^{\beta d}}_{\rightarrow e^{-d_e}} \underbrace{\left ( \frac{1}{d} \right )^{d_e}}_{=d^{-d_e}} \underbrace{\left ( \frac{\beta d}{\beta d-d_e} \right )^{d_e}}_{\rightarrow 1} \underbrace{\frac{r(\beta d-d_e)}{r(\beta d)}}_{\rightarrow 1}\\
    &= \frac{d!}{(d-d_e)!d^{d_e}}
\end{align}

where the following holds for the error term $r(x)$ of the Stirling approximation~\cite{robbins1955remark}:
\[
    \exp \left ( \frac{1}{12x+1} \right ) \leq r(x) \leq  \exp \left ( \frac{1}{12x} \right ).
\]

Then, $\frac{r(\beta d-d_e)}{r(\beta d)} \rightarrow 1$ for $\beta \rightarrow \infty$ holds since
\begin{align}
    \frac{r(\beta d-d_e)}{r(\beta d)} &\leq \exp \left ( \frac{1}{12(\beta d-d_e)}-\frac{1}{12\beta d+1} \right )\\
    &= \exp \left ( \frac{12d_e+1}{12^2\beta^2d^2+12\beta d -12^2\beta d d_e -12d_e } \right ) \\
    &= \exp \left (  \frac{d_e+\frac{1}{12}}{\beta d(12\beta d+1-12d_e)-d_e} \right ),
\end{align}
and
\begin{align}
    \frac{r(\beta d-d_e)}{r(\beta d)} &\geq \exp \left ( \frac{1}{12(\beta d-d_e)+1}-\frac{1}{12\beta d} \right )\\
    &= \exp \left ( \frac{12d_e-1}{12^2\beta^2d^2+12\beta d-12^2\beta d d_e-12\beta d_e} \right ) \\
    &= \exp \left (  \frac{d_e-\frac{1}{12}}{\beta d(12\beta d +1 -12d_e)-d_e} \right ),
\end{align}
which both go to $1$ as $\beta \rightarrow \infty$.

Hence, \thesbo' worst-case success probability is bounded from below and above by sequences that converge to the same point as $D\rightarrow \infty$.
The squeeze theorem~(e.g.,~\cite{sohrab2003basic}) implies
\[
    \lim_{D\rightarrow \infty} p_B(Y^*; D, d ,d_e) = \frac{d!}{(d-d_e)!d^{d_e}}.
\]
\end{proof}

\section{Consistency of \thesbo}\label{app:global_convergence}
We prove the global convergence of function values for \thesbo.
The proof idea is similar to \citet{eriksson2021scalable} but relaxes the assumption of a unique global minimizer.
By construction, $f$ is sparse (see Definition~\ref{def:function-with-active-subspace}), i.e., there exists a set of dimensions of $f$ that do not influence the function value.
Thus, an optimal solution stays optimal regardless of how inactive dimensions are set.
This is why we must relax the assumption of a unique global minimizer in the input space.
Instead, we assume a unique global minimizer in the active subspace $\bm{z}^*\in \mathcal{Z}$ that can map to arbitrarily many minimizers in the input space.
Note that this assumption covers the case when the target space corresponds to the input space, i.e., $d=D$.

\begin{theorem}[\textbf{\thesbo consistency}]
With the following definitions:
\renewcommand{\theenumi}{D\arabic{enumi}}
\begin{enumerate}
    \item $\{\bm{x}_k\}_{k=1}^{\infty}$ is a sequence of points of decreasing function value;
    \label{d1}
    \item $\bm{x}^*\in \argmin_{\bm{x}\in \mathcal{X}}f(\bm{x})$ is a minimizer in $\mathcal{X}$;
    \label{d2}
\end{enumerate}
and under the following assumptions:
\renewcommand{\theenumi}{A\arabic{enumi}}
\begin{enumerate}
    \item $D$ is finite; \label{enum:D_finite}
    \item $f$ is observed without noise; \label{enum:f_noiseless}
    \item $f$ is sparse and bounded in $\mathcal{X}$, i.e., $\exists\, C\in \mathbb{R}_{++}$ s.t. $|f(\bm{x})|<C\, \forall \,\bm{x}\in \mathcal{X}$; \label{enum:f_bounded}
    \item At least one of the minimizers $\bm{x}^*_i$ lies in a continuous region with positive measure; \label{enum:continuous_region}
    \item Once \thesbo reached the input dimensionality $D$, the initial points $\{\bm{x}_i\}_{i=1}^{n_{\textrm{init} }}$ after each TR restart for \thesbo are chosen such that $\forall \delta \in \mathbb{R}_{++}$ and $\bm{x}\in \mathcal{X}$, $\exists\, \nu(\bm{x}, \delta)>0$: $\mathbb{P}\left (\exists i: ||\bm{x}-\bm{x}_i ||_2\leq \delta \right )\geq \nu(\bm{x}, \delta)$, i.e., the probability that at least one point in $\{\bm{x}_i\}_{i=1}^{n_{\rm init}}$ ends up in a ball centered at $\bm{x}$ with radius $\delta$ is at least $\nu(\bm{x},\delta)$; \label{enum:ball_centered}
\end{enumerate}
\renewcommand{\theenumi}{\arabic{enumi}}
$f(\bm{x}_k)$ converges to $f(\bm{x}^*)$ with probability 1.
\end{theorem}
\begin{proof}
We first show that \thesbo must eventually arrive at an embedding equivalent to the input space.
By Assumption~\ref*{enum:D_finite}, the number of accepted ``failures'' (i.e., the number of times \thesbo needs to fail in finding a better solution until the TR base length is shrunk) is always finite since it is always bounded by the target dimension ($\forall i\, \tau_{\rm fail}^i\leq d_i$) which is at most equal to $D$ ($\forall i\, d_i\leq D$).
By the facts that \thesbo considers any sampled point an improvement only if it improves over the current best solution by at least some constant $\gamma\in\mathbb{R}_{++}$ and that $f$ is bounded (Assumption~\ref*{enum:f_bounded}), \thesbo can only perform a finite number of function evaluations without increasing the target dimensionality of its embedding.

Once \thesbo reaches $D$, it behaves like \turbo~\cite{eriksson2019scalable} for which \citet{eriksson2021scalable} proved global convergence assuming a unique global minimizer.
For the case $d_e<D$, we notice that multiple minima in the input space occur due to inactive dimensions that do not influence the function value.

The remainder of our proof is based on the convergence theorem for global search by~\citet{solis1981minimization}, which proves convergence of function values for random search with possibly multiple minima.
By considering the sequence
\[
    \left \{ \bm{x'}_i \in \argmin_{\bm{\hat{x}}\in \{\bm{x}_k\}_{k=1}^i}f(\bm{\hat{x}})\right \}_{i=1}^{\infty}
\]
of points of decreasing function values where $\{\bm{x}_k\}_{k=1}^i$ are the observations up to the $i$-th function evaluation, Definition~\ref*{d1} is satisfied.
Additionally, by the fact that, at each TR restart, \thesbo performs random restarts with uniform probability on $\mathcal{X}$, \thesbo satisfies the assumptions of the theorem by \citet{solis1981minimization}.

The theorem by~\citet{solis1981minimization} states that for a sequence $\{\bm{x}_k\}_{k=1}^{\infty}$ of sampling points with $\varepsilon \in \mathbb{R}_{++}$,
\begin{align*}
    & \lim_{k \rightarrow \infty} \mathbb{P}\left [ \bm{x}_k \in R_{\varepsilon} \right ]=1\\
    & R_{\varepsilon} = \{\bm{x}\in \mathcal{X} : f(\bm{x})<\alpha+\varepsilon\} \\
    & \alpha = \inf\{t: v(\bm{x}\in \mathcal{X}:f(\bm{x})<t)>0\}
\end{align*}
where $R_{\varepsilon}$ is the set of $\varepsilon$-optimal function values, $\alpha$ is the \emph{essential infimum}, and $v$ is the Lebesgue measure.
Note that the essential infimum $\alpha$ is equal to the minimum if the minimizer lies in a continuous region of positive measure, i.e., $\alpha = f(\bm{x}_i^*)$.
By Assumption~\ref*{enum:continuous_region} and by letting $\varepsilon \rightarrow 0$, $f(\bm{x}_k)$ converges to  $f(\bm{x}_i^*)$.
\end{proof}

\section{Additional empirical evaluations}

\subsection{Ablation study for the \thesbo embedding}\label{app:ablation_axus}
We conduct an ablation study to investigate the difference between the \thesbo and \hesbo embeddings.
We run \turbo of \citet{eriksson2019scalable} in an embedded subspace with the two different embeddings.
We use a version of \textsc{Ackley10} (ten active dimensions, i.e., $d_e=10$), where we shift the optimum away from the origin with a uniformly random vector $\bm{\delta} \in [-32.768, 32.768]^{d_e}$ with $\delta_i\sim \mathcal{U}(-32.768, 32.768)$.
The function we optimize is then
\begin{align*}
    f_{\textrm{ShiftedAckley10} }(\bm{x}) = f_{\textrm{Ackley10} }(\bm{x}+\bm{\delta}).
\end{align*}
We adjust the boundaries of the search space such that $f_{\textrm{ShiftedAckley10} }$ is evaluated on the domain $\mathcal{X}=[-32.768, 32.768]^{d_e}$.
The reason for shifting the optimum is that the original \textsc{Ackley} function has its optimum at the origin.
In that case, any \sparseembedding contains this optimum, even if all the active input dimensions are mapped to the same target dimension.

We add 10 dummy dimensions, such that $D=30$ and set $d=20$.
With this problem-setting, the \thesbo and \hesbo embeddings have a probability of approximately 0.27 and 0.07 of containing the optimum, respectively.

\begin{figure}[H]
    \centering
    \resizebox{0.8\textwidth}{!}{\includeinkscape{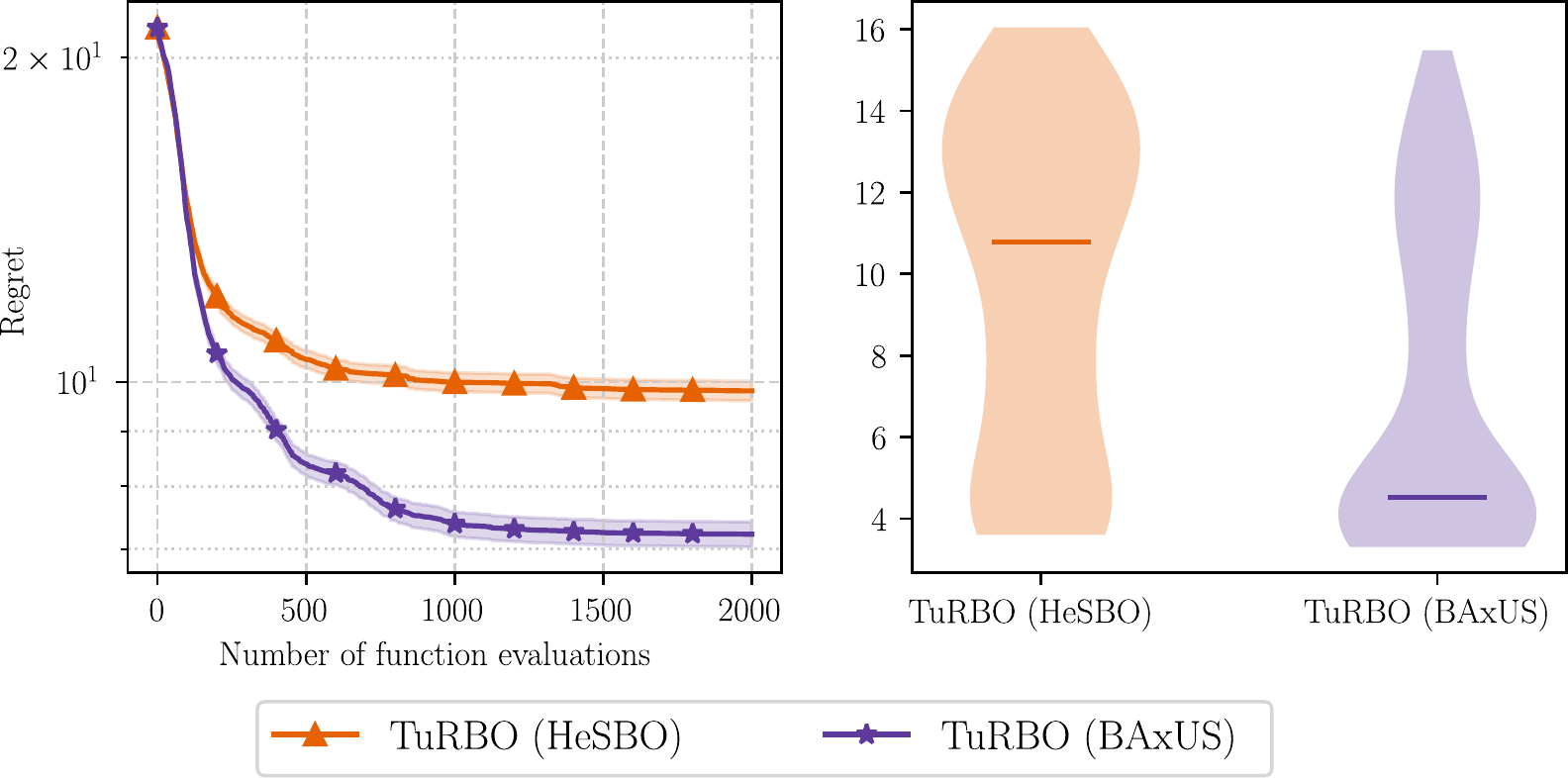}}
    \caption{\textbf{Left}: the \thesbo embedding gives better optimization performance on the shifted \textsc{Ackley10} function: \turbo in embedded subspaces of the \thesbo and \hesbo embeddings. The \thesbo embedding has a higher probability to contain the optimum. \textbf{Right}: the distribution of the final incumbents (lower the better). The horizontal bars show the median.}
    \label{fig:ablation}
\end{figure}

The left side of Figure~\ref{fig:ablation} shows the incumbent mean for \turbo in the two different embedded subspaces.
The shaded regions show one standard error.
\turbo in an \thesbo embedding has significantly better optimization performance than in a \hesbo embedding.
The right side of Figure~\ref{fig:ablation} shows the distributions of the final incumbents and their median.
The \thesbo embedding leads to a significantly lower median and only rarely a similarly bad embedding as the \hesbo method when combined with \turbo.

We perform a two-sided Wilcoxon rank-sum statistical test to check the difference between the best observed function values for the two embeddings.
The difference is significant with $p\approx 0.00001$.

The performance difference between the two embeddings depends on the characteristics of the function and the different dimensionalities, the input dimensionality $D$, the target dimensionality $d$, and the effective dimensionality $d_e$.
For problems with few active dimensions and many input dimensions, the \thesbo and \hesbo embeddings become more similar (see Figure~\ref{fig:success_probabilities}).
However, by Corollary~\ref{corr:optimality}, the \thesbo embedding is, in expectation and in terms of the worst-case success probability, better than the \hesbo embedding for arbitrary sparse functions.

For functions with an optimum at the origin, both embeddings contain that optimum regardless of $d$: Even if all active input dimensions are mapped to the same target dimension, the optimum in the input space can be reached by ``setting'' this particular target dimension to zero.

\paragraph{\turbo with \thesbo embedding vs. \thesbo.}
We compare the simple idea of running \turbo in a \thesbo embedding of fixed target dimensionality with the \thesbo algorithm described in Section~\ref{sec:algorithm}.
We run this simple approach for $11$ different target dimensionalities $d$ ($2,10,20,\ldots,100$) on the \textsc{Lasso-Hard} benchmark and show the results with a sequential color map in Figure~\ref{fig:thesbo_vs_baxus}.
Only the first $d=2$-dimensional embedding achieves the same initial speedup as \thesbo, which is expected as \thesbo starts in a similarly low-dimensional initial embedding.
However, the fixed embedding cannot explore the input space sufficiently and has the worst final solution.
High-dimensional fixed embeddings have more freedom in exploring the input space;
however, they suffer from slower initial optimization performance.

\thesbo has the same initial speedup as the two-dimensional fixed embedding but can explore the space further by increasing the dimensionality of its embedding.
\begin{figure}[H]
    \centering
    \resizebox{0.5\textwidth}{!}{\includeinkscape{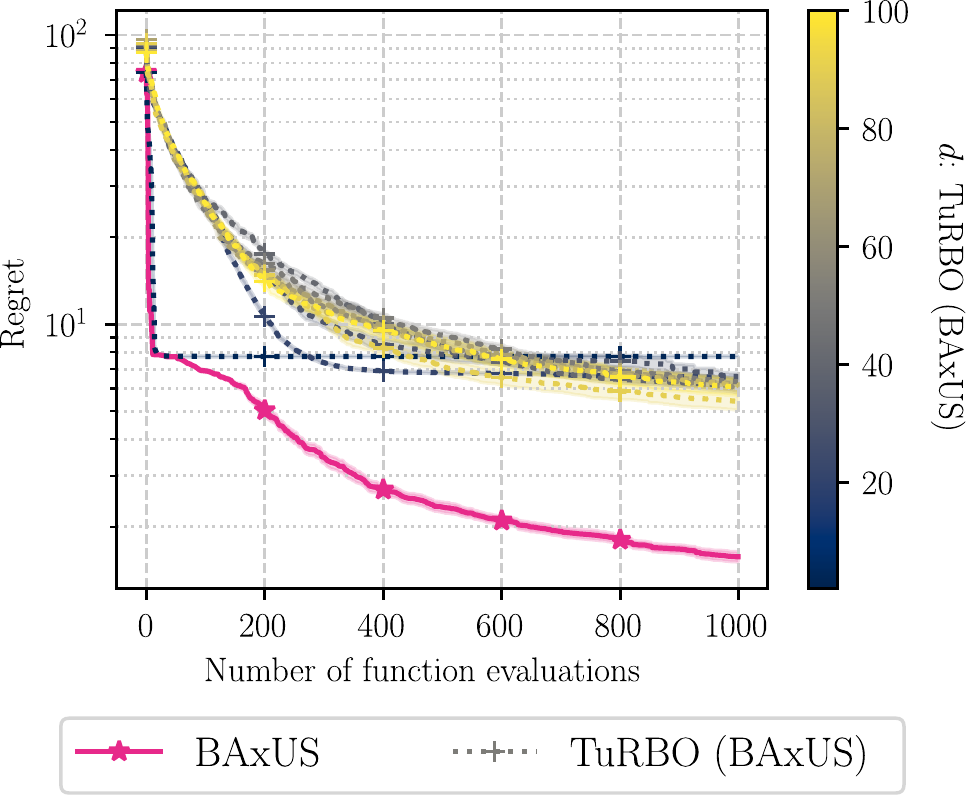}}
    \caption{An evaluation of \thesbo and \turbo with \thesbo embeddings of different target dimensionalities on \textsc{Lasso-Hard}: We run \turbo with the \thesbo embedding for fixed target dimensionalities $d=2,10,20,\ldots,100$ and compare to \thesbo.}
    \label{fig:thesbo_vs_baxus}
\end{figure}

Summing up, we observe that \thesbo achieves a better performance than \turbo with a fixed embedding dimensionality.

\subsection{Evaluation on an additional Lasso benchmark}\label{app:lasso-dna}
In addition to the synthetic \textsc{Lasso-High} and \textsc{Lasso-Hard} benchmarks studied in Section~\ref{sec:results}, we evaluate \thesbo on the \textsc{Lasso-DNA} benchmark from \textsc{LassoBench}~\cite{vsehic2021lassobench}.
The \textsc{Lasso-DNA} benchmark is a biomedical classification task, taking binarized DNA sequences as input~\cite{vsehic2021lassobench}.
\begin{figure}[H]
    \centering
    \resizebox{.7\textwidth}{!}{\includeinkscape{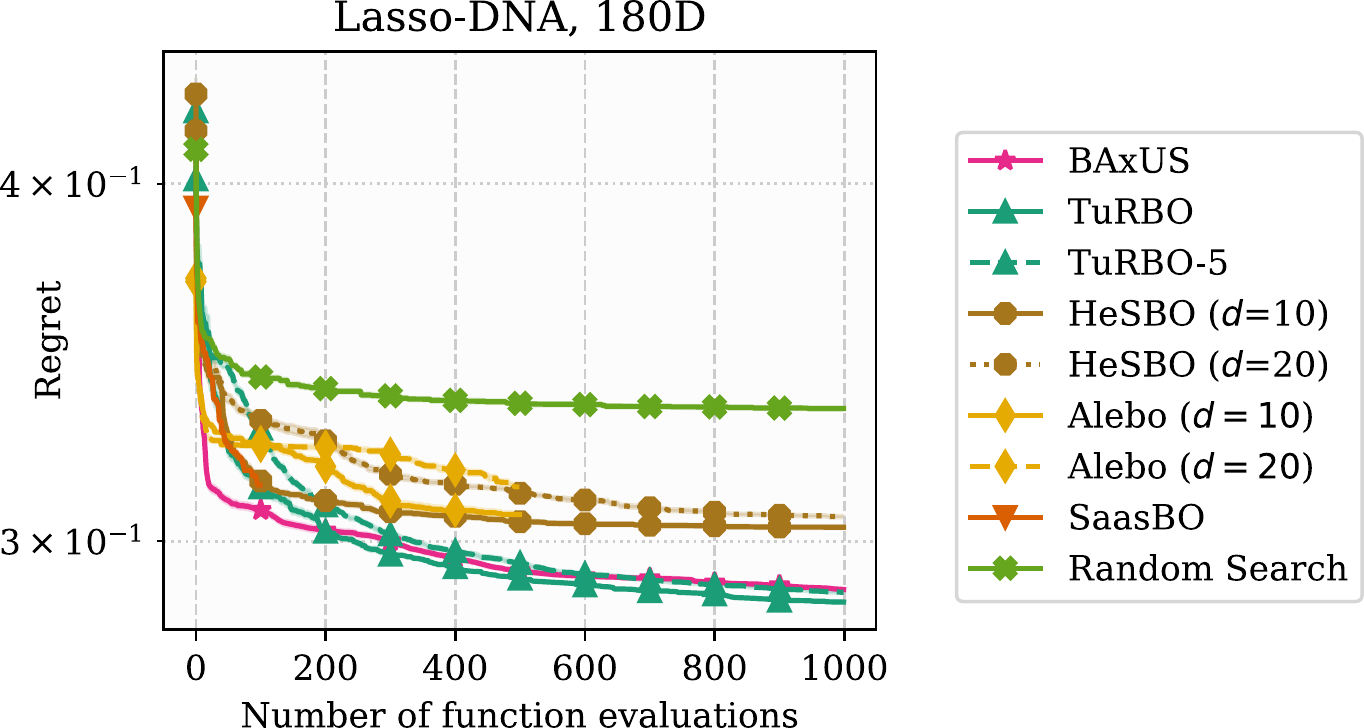}}
    \caption{
        \thesbo and baselines on \textsc{Lasso-DNA}.
        As before, \thesbo makes considerable progress in the beginning and converges faster than \turbo and \cmaes.}
    \label{fig:lasso-dna}
\end{figure}
Figure~\ref{fig:lasso-dna} shows the mean performance of \thesbo on the \textsc{Lasso-DNA}.
Each line shows the incumbent mean;
the shaded regions around the lines show one standard error.
We see the same qualitative behavior as discussed in Section~\ref{sec:results}:
\thesbo reaches a good initial solution faster than any other method.%

After a worse start, \turbo finds slightly better solutions than \thesbo.

\subsection{Evaluation on additional MuJoCo benchmarks}\label{app:mujoco}
We evaluate \thesbo with the same baselines as in Section~\ref{sec:results}.
We use the implementation of~\cite{wang2020learning}\footnote{\url{https://github.com/facebookresearch/LA-MCTS/blob/main/example/mujuco/functions.py}, last accessed: 06/10/2022}, in particular we use the \texttt{Gym} environments \texttt{Ant}, \texttt{Swimmer}, \texttt{Half-Cheetah}, \texttt{Hopper}, \texttt{Walker 2D}, and \texttt{Humanoid 2D}, all in version 2.
For the 6392-dimensional \texttt{Humanoid} benchmark, we limit the target dimensionality of \thesbo to 1000 dimensions to keep the split budgets sufficiently large.
For the other benchmarks, we do not limit the target dimensionality.
Due to the high variance between runs, we ran all methods for 50 different runs.

We summarize the results in Fig.~\ref{fig:mujoco}. 
We observe that \thesbo obtains equal or better solutions than the competitors on four out of six benchmarks.
On the 120-dimensional \texttt{Walker} benchmark, \thesbo is the clear winner, followed by \turbo and \cmaes.
On the 888-dimensional \texttt{Ant} benchmark, \hesbo finds the best solutions, followed by \thesbo that outperforms \turbo and \cmaes.
For the 102-dimensional \texttt{Half-Cheetah}, \turbo produces the best solutions, followed by \cmaes and \thesbo; here, the subspace-based approaches (\alebo and \hesbo) find significantly worse solutions.
For the 6392-dimensional \texttt{Humanoid 2D}, \cmaes obtains the best solutions, followed by \thesbo, \alebo, and \hesbo.
\begin{figure}
    \centering
    \resizebox{\textwidth}{!}{\includeinkscape{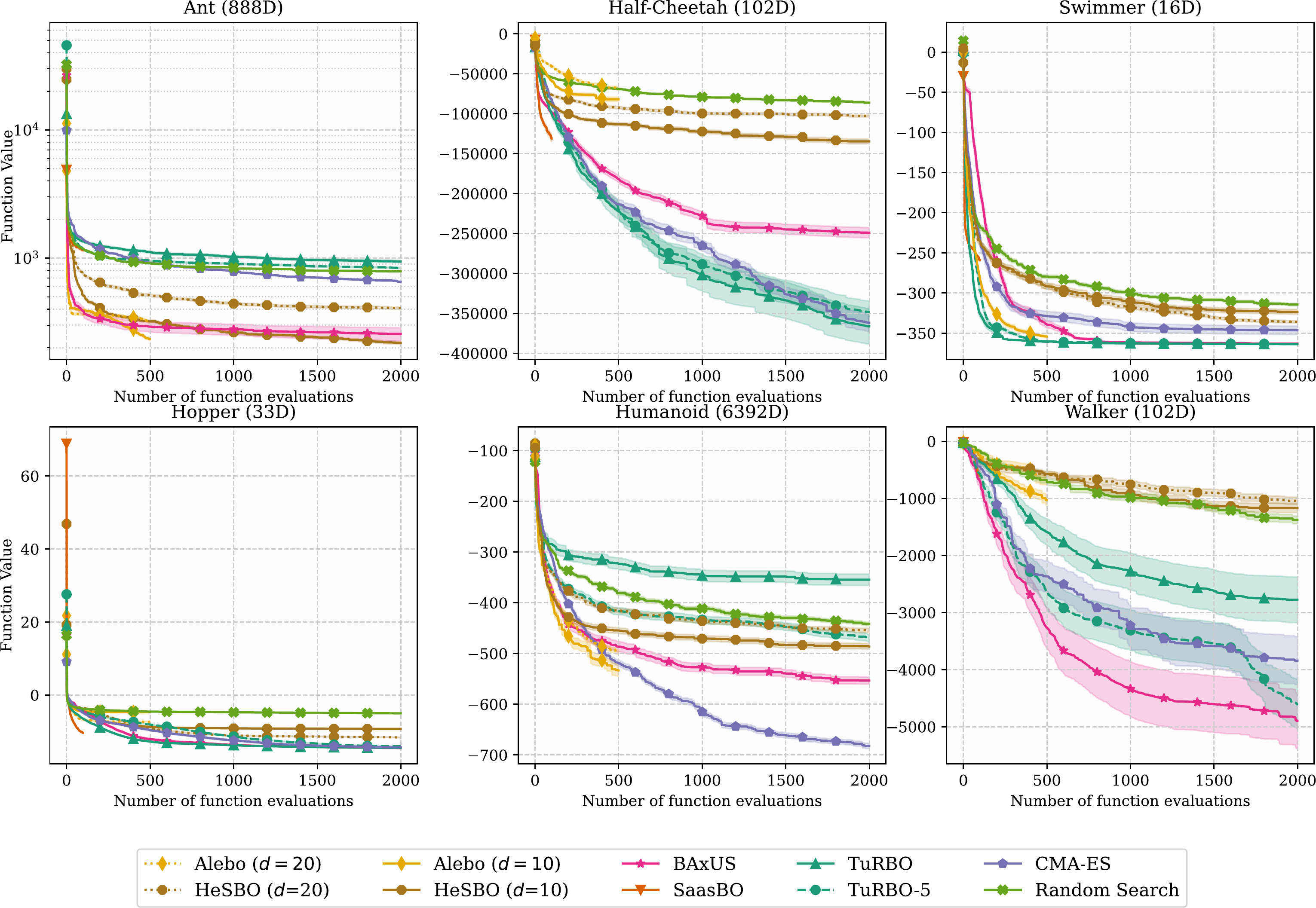}}
    \caption{An evaluation of \thesbo and other methods on high-dimensional test problems of MuJoCo.
    }
    \label{fig:mujoco}
\end{figure}

\section{The nested family of random embeddings}\label{app:opei}
We describe the method for increasing the target dimensionality under the retention of the observations.
Suppose that we have collected $n$ observations and are in target dimension $d$ when Algorithm~\ref{alg:opei} is invoked.
Algorithm~\ref{alg:opei} loops over the target dimensions $1,\ldots,d$.
For each target dimension, the contributing input dimensions are randomly re-assigned to new bins of given sizes.
This can, for example, be realized by first randomly permuting the list of contributing input dimensions, and then dividing the list into $b+1$ chunks (bins).
If the number of contributing input dimensions is less than $b+1$ (remember that $b$ is the number of \emph{new} bins), then it is not possible to re-assign the contributing input dimensions to $b+1$ bins.
Therefore, we re-assign the contributing input dimensions to $\hat{b}=\min(b, l_s-1)$ new bins, where $l_s$ is the number of contributing input dimensions to the $s$-th target dimension.
This also ensures that the target dimension never grows larger than $D$ in the \thesbo embedding.
We evenly distribute the $l_s$ contributing input dimensions across the $\hat{b}$ bins by again using the \thesbo embedding.
This gives a smaller (in terms of number of rows) projection matrix $\Tilde{S}^\T$ which we finally use to update $S^\T$:

\begin{algorithm}[H]
\caption{Observation-preserving embedding increase}\label{alg:opei}
\begin{algorithmic}
\Require transposed embedding matrix $S^\T$, number of new bins per latent dimension $b$, observed points $Y\in [-1, 1]^{n\times d}$.
\Ensure updated transposed embedding matrix $S^\T$ and updated observation matrix $Y$
\For{$s \in \{1,\ldots,d\}$}
    \State $\bm{D}^s \gets $ contributing input dimensions of $s$-th latent dimension of the current embedding
    \State $l_s\gets |\bm{D}^s|$
    \State $\hat{b}\gets \min(b,l_s-1)$ \Comment{If $l_s-1<b$, we can at most create $l_s-1$ new bins.}
    \State Copy and append $s$-th column of $Y$ $\hat{b}$ times at the end of $Y$.
    \State Add $\hat{b}$ zero columns at the end of $S^\T$.
    \State $\bm{\sigma} \gets $  signs of dimensions $\in \bm{D}^s$.
    \State $\Tilde{S}^\T \gets$ \texttt{Baxus-Embedding}($l_s$, $\hat{b}+1$) \Comment{Re-assign input dims. equally\footnotemark, $\Tilde{S}^\T \in \{0,\pm 1\}^{l_s\times \hat{b}+1}$}
    \For{$i\in \{1,\ldots,l_s\}, j\in \{1,\ldots,\hat{b}+1\}$}
        \If{$\Tilde{S}^\T_{ij}\neq 0$}
            \If{$j>1$} \Comment{Move values that fall into new bins to end of $S^\T$.}
                \State $S^\T_{\bm{D}^s_{i},\hat{d}-\hat{b}-1+j}\gets \bm{\sigma}_i$ \Comment{$\hat{d}$: columns of $S^\T$}
                \State $S^\T_{\bm{D}^s_{i},s}\gets 0$ \Comment{Set value in ``old'' column to zero.}
            \EndIf
        \EndIf
    \EndFor
\EndFor
\State \textbf{Return} $S^\T$ and $Y$.
\end{algorithmic}
\end{algorithm}
\footnotetext{Equally means that all $\hat{b}+1$ bins have roughly the same number of contributing input dimensions. The number of contributing input dimensions to the different bins differ by at most 1.}

\section{Additional details on the implementation and the empirical evaluation}\label{append:implementation_details}
We benchmark against \saasbo, \turbo, \hesbo, \alebo, and \cmaes:
\begin{itemize}[leftmargin=*]
\item  For \saasbo, we use the implementation from~\cite{saasbo} (\url{https://github.com/martinjankowiak/saasbo}, license: none, last accessed: 05/09/2022).
\item For \turbo, we use the implementation from~\cite{eriksson2019scalable} (\url{https://github.com/uber-research/TuRBO}, license: \texttt{Uber}, last accessed: 05/09/2022).
\item For \hesbo and \alebo, we use the implementation from~\cite{alebo} (\url{https://github.com/facebookresearch/alebo}, license: \texttt{CC BY-NC 4.0}, last accessed: 05/09/2022).
\item For the \textsc{Lasso} benchmarks, we use the implementation from~\cite{vsehic2021lassobench} (\url{https://github.com/ksehic/LassoBench}, license: \texttt{MIT} and \texttt{BSD-3-Clause}, last accessed: 05/09/2022).
\end{itemize}
We use \texttt{GPyTorch} (version 1.8.1) to train the \ac{GP} with the following setup:
We place a top-hat prior on the Gaussian likelihood noise, the signal variance, and the length scales of the Matérn 5/2 ARD kernel.
The interval for the noise is $[0.005, 0.2]$, for the signal variance $[0.05, 20]$, and for the lengthscales $[0.005, 10]$. 

We evaluate on the synthetic \textsc{Branin2}\footnote{See \url{https://www.sfu.ca/~ssurjano/branin.html}, last accessed: 05/09/2022} and \textsc{Hartmann6}\footnote{See  \url{https://www.sfu.ca/~ssurjano/hart6.html}, last accessed: 05/09/2022} functions.
Since we augment the function with dummy dimensions, we use the same domain for $x_1$ and $x_2$, namely $[-5,15]^D$ for \textsc{Branin2} and $[0,1]^D$ for \textsc{Hartmann6}.

Similar to \turbo, we sample a $\min(100d_n, 5000)$-element Sobol sequence on which we minimize the posterior sample.
To maximize the marginal log-likelihood of the \ac{GP}, we sample 100 initial hyperparameter configurations.
The ten best samples are further optimized using the \textsc{Adam} optimizer for $50$ steps.

We ran the experiments for approximately $15{,}000$ core hours on Intel Xeon Gold 6130 CPUs provided by a compute cluster.

\end{document}